\documentclass[twocolumn]{svjour3}         % twocolumn

\smartqed  % flush right qed marks, e.g. at end of proof

\usepackage{graphicx}
\usepackage{mathptmx}      % use Times fonts if available on your TeX system
\usepackage{amsmath,amssymb,amsfonts}
\usepackage{mathtools}
\usepackage{subfigure}
\usepackage{hyperref}

\begin{document}

\title{Image sequence interpolation using optimal control \thanks{
    This work is supported by the Zentrale Forschungsf\"orderung,
    Universit\"at Bremen within the PhD group ``Scientific Computing in
    Engineering'' (SCiE).
  }
}
%\subtitle{}

\author{Kanglin Chen         \and
        Dirk A. Lorenz %etc.
}


\institute{Kanglin Chen \at
              SCiE, ZeTeM, University of Bremen, Bibliothekstra\ss e 1, 28359 Bremen, Germany, (+49)421 218-63808 \\
              \email{kanglin@math.uni-bremen.de}           %  \\
           \and
           Dirk A. Lorenz \at
	   Institute for Analysis and Algebra, TU Braunschweig, Pockelsstra\ss e 14 - Forum, 38092 Braunschweig, Germany, (+49)531 391-7423 \\
	      \email{d.lorenz@tu-braunschweig.de}
}

\date{Received: date / Accepted: date}

\maketitle
\begin{abstract}
  The problem of the generation of an intermediate image between two
  given images in an image sequence is considered. The problem is
  formulated as an optimal control problem governed by a transport
  equation. This approach bears similarities with the Horn \&
  Schunck method for optical flow calculation but in fact the model is
  quite different. The images are modelled in $BV$ and an analysis of
  solutions of transport equations with values in $BV$ is
  included. Moreover, the existence of optimal controls is proven and
  necessary conditions are derived. Finally, two algorithms are given
  and numerical results are compared with existing methods. The new
  method is competitive with state-of-the-art methods and even
  outperforms several existing methods.
\end{abstract}

\section{Introduction}
\label{sec:1}
Image sequence interpolation is the generation of intermediate images between two given images containing some reasonable motion fields. 
It is mainly based on motion estimation and has broad applications in the area of video compression. In video compression, the knowledge of motions helps remove the non-moving parts of images and compress video sequences with high compression rates. For example in the MPEG format, motion estimation is the most computationally expensive portion of the video encoder and normally solved by mesh-based matching techniques, e.g. blocking matching, gradient matching \cite{watkin04}. While decompressing a video intermediate images are generated by warping the image sequence with motion vectors.

Another possibility of image interpolation is based on optical flow estimation. Since Horn and Schunck proposed the gradient-based method for optical flow estimation in their celebrated work \cite{horn81}, this field has been widely developed till now. For example, instead of the linear constraint in the Horn \& Schunck method one applies the non-linear isotropic constraints \cite{aubert99,bruhn89}, anisotropic diffusion constraints \cite{nagel83,enkelmann88} and TV constraint \cite{wedel09} for preserving the flow edges, which is very useful for motion segmentation.  Dealing with large displacements in image sequences one develops warping technique \cite{brox04} to estimate the flow field in a robust way. However, in \cite{hinterberger01} is shown that the Horn \& Schunck method is only suited for optical flow estimation, but not for matching image intensities, especially in case of large displacements, see also the argumentation in~\cite{stich08}.

Borz{\'{i}}, Ito and Kunisch considered  the optical flow problem in the optimal control 
framework \cite{borz02}. Due to an optimal control formulation the estimated flow field is also suitable for image interpolation, since one searches the flow field such that the interpolated  image has a best matching to a given image in the sense of some norm.  In this paper we modify the model proposed in \cite{borz02} for interpolating intermediate images between two given images and analyze the well-posedness of the corresponding minimizing problem. In the end we introduce an efficient numerical method for solving the optimality system and we also propose a modification of the segregation loop of the optimality conditions system, which give better interpolation results and is robust with respect to the choice of regularization parameter.  To evaluate our proposed interpolation methods we will utilize the image database generated by Middlebury College \footnote{\url{http://vision.middlebury.edu/flow/data/}} and compare our results using the evaluation method of Middlebury with the results in \cite{stich08}.

%---------------------------------------------------------Model------------------------------------------------
\section{Modeling}
\label{sec:2}
We are interested in finding a flow field, which is suitable for image matching. It means that instead of minimizing the optical flow constraint equation directly,  we utilize the transport equation to fit a given image $u_{0}$ to another given image $u_{T}$ in the sense of some predefined norm in the cost functional.

Let us model the optimal control problem governed by the transport equation. Consider the Cauchy problem for the transport equation in $[0,T]\times\Omega$, $\Omega\subset\mathbb{R}^{d}$ (generally $d=2$):
\begin{equation}
\label{eq:transport}
\left\{
\begin{array}{l}
\partial_{t}u(t,x)+b(t,x)\cdot\nabla_{x}u(t,x)  =  0\quad\text{in }]0,T]\times\Omega, \\
\\
u(0,x) = u_{0}(x)\quad\text{in }\Omega.
\end{array}
\right.
\end{equation}
Here $b:[0,T]\times\Omega\longrightarrow\mathbb{R}^{d}$ is an optical flow field, $u_{0}$ is a given initial condition and $u$ is an unknown function depending on $t$ and $x$. We define the nonlinear solution operator of $\eqref{eq:transport}$
\begin{eqnarray*}
G:X\times Y &\longrightarrow& Z,\\
(u_{0},b)&\mapsto& u,
\end{eqnarray*}
where $X,Y,Z$ are normed spaces to be specified. Then, we define a linear \textquotedblleft observation operator\textquotedblright~$E_{T}:u\mapsto u(T)$, which observes the value of $u$ at time $T$. By the chain $(u_{0},b)\mapsto u\mapsto u(T)$ we have the ``control-to-state mapping''
\begin{eqnarray*}
S:X\times Y &\longrightarrow& U,\\
S:(u_{0},b)&\mapsto&u(T).
\end{eqnarray*}
The space $U$ is a subspace of $Z$, which not involves time $t$. The continuity of $S$ will be investigated in the concrete contexts. 
Our intention is to find the flow field $b$ such that the corresponding image $S(u_{0},b)$ matches the image $u_{T}$ at time $T$ as well as possible. This motivates to minimize the functional $\frac{1}{2}\left\|S(u_{0},b)-u_{T}\right\|^{2}_{U}$. However, this problem is ill-posed and an additional regularization term is needed. This regularized optimal control problem can be formulated as minimizing the following cost functional
\begin{equation}
\label{eq:opc}
\inf\limits_{b\in Y}J(b)=\frac{1}{2}\left\|S(u_{0},b)-u_{T}\right\|^{2}_{U}+\frac{\lambda}{2}\left\|b\right\|^{2}_{Y},\\
\end{equation}
\begin{equation}
\label{eq:divergence-free}
\text{subject to}\quad\mathrm{div}b=0.
\end{equation}
We use Tikhonov regularization to stabilize the cost functional and $\lambda$ is the regularization parameter. In the framework of optimal control \cite{lions71,fred05} we call $b$ the control and $u$ the state. According to the conservation law \cite{hirsch07} and the divergence theorem \cite{riley06}, the divergence free constraint of $b$ will make the  flow volume conserving, smooth and vary not too much inside the flow field of a moving object. Such properties are desired to be enjoyed in image interpolation in case that the moving objects are not getting deformed. Such constraint is not new for optical flow estimation and was similarly introduced as a regularization constraint e.g. in \cite{suter94,kameda07,borz02}. 

We emphasize, that our model is considerably different from the Horn
\& Schunck approach which is based on the optical flow
constraint. There one has a given image $u$ and a given derivative
$\partial_t u$ (both at time $t_0$) and one finds a flow field $b =
(v,w)$ by minimizing
\[
\int_{\Omega} (\partial_t u - b\cdot\nabla u)^2dt + \int_\Omega
|\nabla v|^2 + |\nabla w|^2 dx.
\]
The main conceptual difference between this approach and ours is that
Horn \& Schunck just consider one time $t_0$ and match the flow field
only to that time. Hence, it is unclear in what sense the produced
field $b$ could be useful to match a given image with another one. Our
approach uses two given images and tries to find a flow field $b$
which transports the first image as close as possible to the second
image. The ``optical flow constraint equation'' now enters as a
constraint to the optimization problem and not in the objective
functional itself.

In next chapter we will give some adequate spaces for $u$ and
$b$. Especially we are interested in images $u_0$ and $u_T$ which are
of bounded variation. Hence, we introduce the solution theory of
transport equations equipped with a smooth flow field and a $BV$ image
as initial value. Especially we need to work out conditions under
which the $BV$-regularity is propagated by the flow field. Then, we
will analyze the existence of a minimizer of problem $\eqref{eq:opc}$
restricted to $\eqref{eq:transport}$ and $\eqref{eq:divergence-free}$.

%----------------------------------------------------------Analysis-----------------------------------------------------

\section{Analysis of Well-posedness}
\label{sec:3}

To analyze the solution operator $G$ we use the method of
characteristics. We start with the analysis of the corresponding ODEs,
then derive existence results for initial values $u_0$ which are of
bounded variation and finally derive a result on the weak sequential
closedness of $G$. Together this shows the existence of an optimal
control in the respective setting.

\subsection{Basic Theory of ODE}
It is well-known that the solution theory of transport equations has a tight relationship with the ordinary differential equation
\begin{equation}
\label{eq:ode}
\left\{
\begin{array}{rcl}
\dot{\gamma}(t) & = & b(t,\gamma(t))\quad t\in I,\\
\\
\gamma(a) & = & x_{0}\quad\text{ in }\Omega.
\end{array}
\right.
\end{equation}
Regarding the solution theory of $\eqref{eq:ode}$, the existence and uniqueness of a solution can be derived by the theorem of Picard-Lindel\"of \cite{hartman02} if $b$ is Lipschitz continuous in space and uniformly continuous in time. We can also relax the assumption on $t$ of $b$ to be integrable by the following Carath\'{e}odory theorem \cite{ambrosio}, which a general version of the  Picard-Lindel\"of theorem:
\begin{theorem}[Carath\'{e}odory]
\label{thm:cara}
Define $I=[a,c]$ and $\Omega$ is a bounded subset in $\mathbb{R}^{d}$. Suppose $b:I\times\Omega\rightarrow\mathbb{R}^{d}$ so that
\begin{enumerate}
\item $t\rightarrow b(t,x)$ is measurable in $I$ for every $x\in\Omega$;
\item there exists $C\geq 0$ with $|b(t,x)-b(t,x')|\leq C|x-x'|$ for a.e. $t\in I$ and every $x,x'\in\bar{\Omega}$;
\item $b(t,x)=0$ for a.e. $t\in I$ and every $x\in\partial\Omega$;
\item the function $m(t)=|b(t,x_{0})|$ is integrable in $I$ for $x_{0}\in\Omega$.
\end{enumerate}
Then, there exists a unique solution $\gamma^{*}:I\rightarrow\Omega$ with
\begin{equation*}
\gamma^{*}(t)=x_{0}+\int\limits_{a}^{t}b(s,\gamma^{*}(s))ds\qquad t\in I
\end{equation*}
to the Cauchy problem $\eqref{eq:ode}$.
\end{theorem}
As a consequence of the proof, the flow $\gamma^{*}(t)$ is absolutely continuous in $[a,c]$. Generally, if we consider the solution in $[0,T]$ with $T>c$, we can restart $\gamma^{*}$ at $(c,\gamma^{*}(c))$ until the unique continuous solution arrives at time $T$. The backward flow is the special case when the time $t$ is smaller than the initial time $a$. 

Next, we want to choose an appropriate function space $Y$ for $b$, which is suitable for the control problem. According to \cite{ambro00} the space of Lipschitz functions is equivalent to $W^{1,\infty}(\Omega)^{d}$, if $\Omega$ is a bounded, convex, open set. According to~\cite{colombini04} lower regularity of the flow field (i.e.~$b\in W^{1,p}$ with $p<\infty$) does not preserve $BV$-regularity. However, the norm in $W^{1,\infty}$ is not well suited as a penalty term since it is difficult to determine the necessary optimality conditions of $b$ equipped with the $L^{\infty}-$norm. Thus, we assume additionally that the domain $\Omega$ enjoys the strong local Lipschitz condition \cite{adams03} and use the fact that $H^{3}_{0}(\Omega)^{d}$ is continuously embedded into $W^{1,\infty}(\Omega)^{d}$ under this assumption, when $\dim(\Omega)=2$. Considering the divergence-free constraint on $b$ we set
\begin{equation*}
H^{3,\mathrm{div}}_{0}(\Omega)^{2}:=\left\{f\in H^{3}_{0}(\Omega)^{2}~\Big|~\mathrm{div}f=0\right\}.
\end{equation*}
Adjusting the assumption on the time of $b$ in Theorem \ref{thm:cara} and previous conditions on $\Omega$ we will assume that 
\begin{itemize}
\item $\Omega\subset\mathbb{R}^{2}$ is a bounded, convex, open set with the strong local Lipschitz condition
\item $b\in L^{2}([0,T];H^{3,\mathrm{div}}_{0}(\Omega)^{2})$
\end{itemize}
throughout the paper unless otherwise stated.
A proper choice for the space $U$ will be discussed in Section~\ref{subsec:ex_minimizer}.

In order to formulate the solution of transport equation in a convenient way, we give the concept of classical flow \cite{crippa07}. 
\begin{definition}
The classical flow of vector field $b$ is a map
\begin{equation*}
\Phi(t,x):[0,T]\times\Omega\longrightarrow\Omega
\end{equation*}
which satisfies
\begin{equation}
\label{eqn:ode}
\left\{
\begin{array}{l}
\dfrac{\partial \Phi}{\partial t}(t,x) = b(t,\Phi(t,x))\quad\text{in }]0,T]\times\Omega,\\
\\
\Phi(0,x) = x\quad\text{in }\Omega.
\end{array}
\right.
\end{equation}
\end{definition}
A helpful property of $\Phi$ will be given in the following corollary.
\begin{corollary}
For every $t\in[0,T]$ the mapping
$\Phi(t,\cdot):\Omega\rightarrow\Omega$
is Lipschitz continuous and a diffeomorphism.
\end{corollary}
\begin{proof}
The injectivity can be derived from the uniqueness of the backward flow: If the flow $\Phi$ starts from two points $x_{1}\neq x_{2}$ and arrives at some $t$ at the same point $\Phi(t,x_{1})=\Phi(t,x_{2})=\bar{x}$, the backward flow starting from $(t,\bar{x})$ will be not unique. Regarding the surjectivity: for every point $y\in\Omega$ one can find a backward flow starting from $(t,y)$
\begin{equation*}
\gamma(t') = y + \int\limits_{t}^{t'}b(s,\gamma(s))ds=x\in\Omega,
\end{equation*}
according to Theorem $\ref{thm:cara}$. In case $t'=0$ yields $\Phi(t,x)=y$.

The Lipschitz regularity of $\Phi$ is easily shown by the Gronwall's lemma. For details we refer to~\cite{crippa07}.

Since the Lipschitz continuity gives only the local $C^{1}$-regularity, the $C^{1}$-regularity of $\Phi(t,\cdot)$ in $\Omega$ one can follow the results in \cite{crippa07}, which states that if $b$ has $C^{1}$-regularity in space, then the flow $\Phi(t,\cdot)$ is also $C^{1}$ in space. In fact, $H^{3}_{0}(\Omega)^{2}$ is continuously embedded into $C^{1}(\bar{\Omega})^{2}$, and hence we derive the statement.\qed
\end{proof}

\subsection{Solution Theory of Transport Equations}
In this subsection we will consider the transport equation with the initial value $u_{0}$ in $BV$. The $BV$ space is a natural space for images, since $BV$ contains the functions with discontinuities along hypersurfaces, i.e.~edges of images~\cite{ambro00}. However, the propagation of $BV$ regularity is a delicate matter. We formulate first the solution of transport equations with a smooth initial value:
\begin{corollary}
\label{cor:transport}
Let $u_{0}\in C^{1}(\Omega)$ and $\Phi$ be a classical flow of vector field $b$. Then the transport equation $\eqref{eq:transport}$ has unique solution
\begin{equation}
\label{eq:classic solution}
u(t,x) = u_{0}\circ\Phi^{-1}(t,\cdot)(x).
\end{equation} 
\end{corollary}

\begin{proof}
Let us test \eqref{eq:transport} along the characteristics denoted by $(t,\Phi(t,x))$
\begin{eqnarray*}
0 &=&\dfrac{\partial u}{\partial t}(t,\Phi(t,x))+b(t,\Phi(t,x))\cdot\nabla u(t,\Phi(t,x)) \\
&=&\dfrac{\partial u}{\partial t}(t,\Phi(t,x))+\dfrac{\partial \Phi}{\partial t}(t,x)\cdot\nabla u(t,\Phi(t,x)) \\
&=&\dfrac{\partial}{\partial t}(u(t,\cdot)\circ\Phi(t,x)).
\end{eqnarray*}
This implies that every solution is constant along the characteristics. Adjusting the initial value we derive 
$\eqref{eq:classic solution}$ is a solution to \eqref{eq:transport} and the uniqueness follows immediately from the uniqueness of flow $\Phi$.\qed
\end{proof}
Equipped with a non-differentiable initial value the classic solution $\eqref{eq:classic solution}$ will not work. Next, we give the definition of the solution of transport equations in the weak sense.
\begin{definition}[Weak solution]If $b$ and $u_{0}$ are summable functions and $b$ is divergence free in space, then we say that a function $u:[0,T]\times\Omega\rightarrow\mathbb{R}$ is a weak solution of \eqref{eq:transport} if the following identity holds for every function $\varphi\in C^{\infty}_{c}([0,T[\times\Omega):$
\begin{equation}
\label{def:weak}
\int\limits_{0}^{T}\int\limits_{\Omega}u\left(\partial_{t}\varphi+b\cdot\nabla\varphi\right)dxdt = -\int\limits_{\Omega}u_{0}(x)\varphi(0,x)dx.
\end{equation}
\end{definition}
In Theorem $\ref{thm:weaksolution}$ it will be shown that $\eqref{eq:classic solution}$ is actually the unique weak solution of $\eqref{eq:transport}$ with $u_{0}\in BV(\Omega)$. Before we are able to deal with the proof, we recall briefly the weak$^{*}$ topology of $BV$ \cite{ambro00,attouch06,aubert02,aubert99}, 
\begin{equation*}
u_{n} \xrightharpoonup[BV(\Omega)]{*}u\quad:\Leftrightarrow\quad u_{n}\xrightarrow[L^{1}(\Omega)]{}u \text{ and }
Du_{n} \xrightharpoonup[\mathcal{M}(\Omega)]{*}Du
\end{equation*}
which possesses convenient compactness properties in the following theorem \cite{ambro00}.
\begin{theorem}
\label{thm:bv-w*}
Let $(u_{n})\subset BV(\Omega)$. Then $(u_{n})$ converges weakly* to $u$ in $BV(\Omega)$ if and only if $(u_{n})$ is bounded in $BV(\Omega)$ and converges to $u$ in $L^{1}(\Omega)$.
\end{theorem}
To prove that $\eqref{eq:classic solution}$ is a weak solution of $\eqref{eq:transport}$ it is common to use the technique of mollifiers \cite{evans92}.
In short, we smooth the initial value with a mollifier $\eta_{\epsilon}$ with variance $\epsilon$, let $\epsilon$ converge to zero and investigate the convergence of the solution with a smooth initial value to a nonsmooth initial value. This will be done in next theorem.
\begin{theorem}
\label{thm:bv}
Assume $u_{0}\in BV(\Omega),\varphi$ and $\varphi^{-1}$ are diffeomorphisms and Lipschitz continuous in $\Omega$. Then, the sequence $((u_{0}*\eta_{\epsilon})\circ\varphi)$ converges to $u_{0}\circ\varphi$ in the weak* topology of $BV(\Omega)$.
\end{theorem}
\begin{proof}
Let us verify first the $L^{1}$-convergence of $(u_{0}*\eta_{\epsilon})\circ\varphi$ and set $\varphi(x)=y$
\begin{eqnarray*}
&&\int\limits_{\Omega}|(u_{0}*\eta_{\epsilon})\circ\varphi(x)-u_{0}\circ\varphi(x)|dx\\
&=&\int\limits_{\Omega}|u_{0}*\eta_{\epsilon}(y)-u_{0}(y)||\det(\nabla\varphi^{-1}(y))|dy\\
&\leq&\left\|u_{0}*\eta_{\epsilon}-u_{0}\right\|_{L^{1}(\Omega)}\left\|\det(\nabla\varphi^{-1})\right\|_{L^{\infty}(\Omega)}.
\end{eqnarray*}
Let $L$ be the Lipschitz constant of $\varphi^{-1}$  i.e. $L=\left\|\nabla\varphi^{-1}\right\|_{L^{\infty}(\Omega)^{4}}$, then $\left\|\det(\nabla\varphi^{-1})\right\|_{L^{\infty}(\Omega)}$ is bounded from above by $2L^{2}$. Together with the approximation property of mollifiers this gives the $L^{1}-$convergence. Regarding the weak$^*$ convergence of Radon measures $\nabla(u_0*\eta_\epsilon)$ we observe that for every $\psi\in C^{\infty}_{c}(\Omega)^{2}$ it holds
\begin{eqnarray}
\label{eq:eq1}
&&\int\limits_{\Omega}\nabla((u_{0}*\eta_{\epsilon})\circ\varphi)\psi dx\nonumber\\
&=&-\int\limits_{\Omega}(u_{0}*\eta_{\epsilon})\circ\varphi\mathrm{div}\psi dx\nonumber\\
&=&-\int\limits_{\Omega}(u_{0}*\eta_{\epsilon})(y)\mathrm{div}(\psi\circ\varphi^{-1}(y))|\det\nabla\varphi^{-1}(y)|dy\nonumber\\
&=&-\int\limits_{\Omega}\int\limits_{\Omega}\eta_{\epsilon}(y-s)u_{0}(s)ds\mathrm{div}(\psi\circ\varphi^{-1}(y))|\det\nabla\varphi^{-1}(y)|dy\nonumber\\
&=&-\int\limits_{\Omega}\int\limits_{\Omega}\eta_{\epsilon}(y-s)\mathrm{div}(\psi\circ\varphi^{-1}(y))|\det\nabla\varphi^{-1}(y)|dyu_{0}(s)ds\nonumber\\
&=&-\int\limits_{\Omega}\eta_{\epsilon}*\left(\mathrm{div}(\psi\circ\varphi^{-1})|\det\nabla\varphi^{-1}|\right)(s)u_{0}(s)ds.
\end{eqnarray}
Since $\varphi^{-1}$ is $C^{1}$ and Lipschitz continuous in $\Omega$, the convolved term belongs to $L^{2}(\Omega)$. Recall that in the two dimensional case $BV(\Omega)$ is continuously embedded into $L^{2}(\Omega)$, then utilizing the approximate property of mollifiers implies that the equation $\eqref{eq:eq1}$ converges to
\begin{eqnarray*}
&&-\int\limits_{\Omega}\mathrm{div}(\psi\circ\varphi^{-1}(s))|\det\nabla\varphi^{-1}(s)|u_{0}(s)ds\\
&\stackrel{\varphi(\xi)=s}{=}&-\int\limits_{\Omega}\mathrm{div}\psi(\xi)u_{0}(\varphi(\xi))d\xi\\
&\stackrel{(*)}{=}&\int\limits_{\Omega}\psi D(u_{0}\circ\varphi)
\end{eqnarray*}
In $(*)$ we applied the Gauss-Green formula for the $BV$ functions \cite{evans92}.\qed
\end{proof}

\begin{remark}
\label{rem:bv}
Under the same assumptions of Theorem $\ref{thm:bv}$ one can derive from Theorem $\ref{thm:bv-w*}$ that $((u_{0}*\eta_{\epsilon})\circ\varphi)$ is uniformly bounded in $BV(\Omega)$ and converges to $u_{0}\circ\varphi$ in $L^{1}(\Omega)$, actually also in $L^{p}(\Omega)$ with $p\leq 2$ due to the approximate property of mollifiers and the fact $BV(\Omega)$ has a continuous embedding into $L^{2}(\Omega)$ in the two dimensional case.
\end{remark}

\begin{lemma}
\label{lem:conti-time}
Assume that $u_{0}\in BV(\Omega)$, $\varphi(t,\cdot)$ and $\varphi^{-1}(t,\cdot)$ are diffeomorphisms  in $\Omega$ for every $t\in [0,T]$ and $\varphi(\cdot,x)$ is absolutely continuous in $[0,T]$ for every $x\in\Omega$. Define 
\begin{equation*}
u_{\epsilon}(t,x)=(u_{0}*\eta_{\epsilon})\circ\varphi(t,x).
\end{equation*} 
Then, $u_{\epsilon}\in C([0,T];BV(\Omega))$.
\end{lemma}

We skip the proof of Lemma $\ref{lem:conti-time}$, since it is a trivial result utilizing the substitution technique introduced in the proof of Theorem $\ref{thm:bv}$. Now, we are able to prove the existence and uniqueness of the weak solution of the transport equation $\eqref{eq:transport}$.
\begin{theorem}
\label{thm:weaksolution}
If $u_{0}\in BV(\Omega)$, then there exits a unique weak solution 
\begin{equation}
\label{eq:weaksolution}
\hat{u}(t,x)=u_{0}\circ\Phi^{-1}(t,\cdot)(x)
\end{equation} 
of $\eqref{eq:transport}$ belonging to $L^{\infty}([0,T];BV(\Omega))$.
\end{theorem}
\begin{proof}
Consider the transport equation with initial value $u_{0}$ convolved with mollifier $\eta_{\epsilon}$
\begin{equation*}
\left\{
\begin{array}{l}
\partial_{t}u(t,x)+b(t,x)\cdot\nabla_{x}u(t,x) = 0\quad\text{ in }]0,T]\times\Omega \\
\\
u(0,x) = u_{0}*\eta_{\epsilon}(x)\quad\text{ in }\Omega.
\end{array}
\right.
\end{equation*}
Corollary \ref{cor:transport} implies that there exists a unique solution $u_{\epsilon}$ of the form
\begin{equation*}
u_{\epsilon}(t,x)=(u_{0}*\eta_{\epsilon})\circ\Phi^{-1}(t,\cdot)(x).
\end{equation*}
Let us define
\begin{equation*}
\hat{u}(t,x)=u_{0}\circ\Phi^{-1}(t,\cdot)(x),
\end{equation*}
where $\hat{u}(t,\cdot)\in BV(\Omega)$ according to Theorem $\ref{thm:bv}$ for every $t\in[0,T]$. Remark $\ref{rem:bv}$ gives that $u_{\epsilon}(t,\cdot)$ converges to $\hat{u}(t,\cdot)$ in $L^{2}(\Omega)$ and $u_{\epsilon}(t,\cdot)$ is uniformly bounded in $BV(\Omega)$. And according to Lemma $\ref{lem:conti-time}$ this yields that $u_{\epsilon}$ is uniformly bounded in $L^{\infty}([0,T];BV(\Omega))$, which is continuous embedded into $L^{2}([0,T];L^{2}(\Omega))$. Hence, there exists a subsequence $(u_{\epsilon_{k}})$ of $(u_{\epsilon})$ such that
\begin{equation}
\label{conv:weak-L2}
u_{\epsilon_{k}}\rightharpoonup\hat{u}\text{ in }L^{2}([0,T];L^{2}(\Omega))
\end{equation}
and $\hat{u}\in L^{\infty}([0,T];BV(\Omega))$. Due to the weak convergence of $u_{\epsilon_{k}}$ in $L^{2}([0,T];L^{2}(\Omega))$, one can derive for every $\varphi\in C^{\infty}_{c}([0,T[\times\Omega)$ it holds that
\begin{center}
\begin{tabular}{ccc}
$\int\limits_{0}^{T}\int\limits_{\Omega} u_{\epsilon_{k}}[\partial_{t}\varphi+b\cdot\nabla\varphi]dxdt$ & $\longrightarrow$ & $\int\limits_{0}^{T}\int\limits_{\Omega}\hat{u}[\partial_{t}\varphi+b\cdot\nabla\varphi]dxdt$
\\
&&\\
$\parallel$ & & $\parallel$ \\
&&\\
$-\int\limits_{\Omega}u_{0}*\eta_{\epsilon_{k}}\varphi(0,x)dx $ & $\longrightarrow$ & $-\int\limits_{\Omega}u_{0}\varphi(0,x)dx$.
\end{tabular} 
\end{center}
The upper convergence is valid since $b\in L^{2}([0,T];L^{2}(\Omega)^{2})$ and thanks to $\eqref{conv:weak-L2}$. The lower convergence can be deduced from the property of approximate identity. The left equality is valid for a smooth initial value and smooth vector field. Hence, all of them imply the right equality.

Regarding the uniqueness of weak solution it is shown in \cite{ambro08} that the continuity equation, which is equal to the transport equation in case $\mathrm{div}b=0$, has a unique solution in the Cauchy-Lipschitz framework, i.e. $b\in L^{1}([0,T];W^{1,\infty}(\mathbb{R}^{d}))$. Definitely, it is also valid under our assumption of $b$. 

Because of the uniqueness of the weak solution the convergence of subsequence $(u_{\epsilon_{k}})$ in the previous proof can be proceeded to the whole sequence $(u_{\epsilon})$.\qed
\end{proof}

\subsection{Existence of a Minimizer}
\label{subsec:ex_minimizer}
The goal of this subsection is to complete the cost functional $\eqref{eq:opc}$ with some reasonable norm and investigate the existence of a minimizer of problem $\eqref{eq:opc}$.
First of all, we give the norm of the penalty term of $\eqref{eq:opc}$ w.r.t. $b$. According to \cite{adams03} an equivalent norm of $H^{3}_{0}$ is
\begin{equation}
\label{norm:H3_0}
\left\|b\right\|_{H^{3}_{0}(\Omega)^{2}}=\left(\sum\limits_{|\alpha|=3}\left\|\partial^{\alpha}b\right\|^{2}_{L^{2}(\Omega)^{2}}\right)^{1/2}.
\end{equation}
We can easily find out that the seminorm 
$(\int_{\Omega}|\nabla\Delta b|^{2}dx)^{1/2}$ is actually another equivalent norm of $H^{3}_{0}(\Omega)^{2}$, since it is equivalent to $\eqref{norm:H3_0}$. For the regularity of $b$ in time we can give the equivalent norm of $L^{2}([0,T];H^{3}_{0}(\Omega)^{2})$
\begin{equation}
\label{norm:h3}
\left\|b\right\|^{2}_{L^{2}([0,T];H^{3}_{0}(\Omega)^{2})}=\int\limits_{0}^{T}
\left\|\nabla\Delta b(t,\cdot)\right\|_{L^{2}(\Omega)^{4}}^{2}dt.
\end{equation}
As discussed above, we assume that $u_0$ and $u_T$ are
$BV$-functions. Hence, $BV$ seems to be a proper choice for the space
$U$. However, since $BV$ is continuously embedded in $L^2(\Omega)$ for
$d=2$ we use $U = L^2(\Omega)$ (we discuss this choice in more detail
in Section~\ref{sec:4}).  Hence, our cost functional is
\begin{equation}
\label{eq:costfunctional}
J(b)=\dfrac{1}{2}\left\|S(u_{0},b)-u_{T}\right\|^{2}_{L^{2}(\Omega)}+\dfrac{\lambda}{2}\int\limits_{0}^{T}\left\|\nabla\Delta b(t,\cdot)\right\|^{2}_{L^{2}(\Omega)^{4}}dt.
\end{equation}

\begin{lemma}
\label{lem:bv-bounded n} 
If $(\varphi_{n})$ and $(\varphi^{-1}_{n})$ are sequences of diffeomorphisms in $\Omega$ and the Jacobian determinant $\det\nabla\varphi_{n}$ is uniformly bounded in $L^{\infty}(\Omega)$ by the upper bound $C$. Then, $((u_{0}*\eta_{\epsilon})\circ\varphi^{-1}_{n})$ is uniformly bounded in $BV(\Omega)$ w.r.t. $n$.
\end{lemma}

\begin{proof}
It is easy to check that $(u_{0}*\eta_{\epsilon})$ is uniformly bounded in $BV(\Omega)$ according to Theorem \ref{thm:bv-w*} and \ref{thm:bv}. Suppose that the upper bound is $\widetilde{C}$. Let us verify first the $L^{1}-$norm by setting \mbox{$y=\varphi_{n}^{-1}(x)$}
\begin{eqnarray*} 
&&\int\limits_{\Omega}|(u_{0}*\eta_{\epsilon})\circ\varphi^{-1}_{n}|dx\\
&=&\int\limits_{\Omega}|u_{0}*\eta_{\epsilon}||\det\nabla\varphi_{n}(y)|dy\\
&\leq&C\int\limits_{\Omega}|u_{0}*\eta_{\epsilon}|dy\\
&\leq&C\widetilde{C}\left\|u_{0}\right\|_{L^{1}(\Omega)}. 
\end{eqnarray*}
Regarding the variation norm by $\left\|u_{0}\right\|_{var(\Omega)}:=\int_{\Omega}|D u_{0}|dx$ 
we have
\begin{eqnarray*}
&&\int\limits_{\Omega}|\nabla(u_{0}*\eta_{\epsilon})\circ\varphi^{-1}_{n}|dx\\
&=&\int\limits_{\Omega}|\nabla(u_{0}*\eta_{\epsilon})(y)||\det\nabla\varphi_{n}(y)|dy\\
&\leq&C\int\limits_{\Omega}|\nabla(u_{0}*\eta_{\epsilon})(y)|dy\\
&\leq&C\widetilde{C}\left\|u_{0}\right\|_{var(\Omega)}.
\end{eqnarray*}\qed
\end{proof}

\begin{lemma}
\label{lem:uniform boundedness}
If $(b_{n})$ is uniformly bounded in $L^{2}([0,T];H^{3}(\Omega)^{2})$ and $u_{0}\in BV(\Omega)$. Define
$u_{n,\epsilon}=(u_{0}*\eta_{\epsilon})\circ\Phi^{-1}_{n}$ and $u^{t}_{n,\epsilon}=u_{n,\epsilon}(t)$. Then, there exists a subsequence $(u_{n_{k},\epsilon})$ such that $u_{n_{k},\epsilon}$ converges to some limit $u_{\epsilon}$ in $L^{2}([0,T];L^{p}(\Omega))$ with $p<2$ and weakly to $u_{\epsilon}$ with $p=2$. $u_{n_{k},\epsilon}^{t}$ converges to  $u_{\epsilon}(t)$ in $L^{p}(\Omega)$ with $p<2$ and weakly to $u_{\epsilon}(t)$ with $p=2$.
\end{lemma}

\begin{proof}
Recall that for every $b_{n}$ there is a corresponding $\Phi_{n}$ s.t. $\Phi_{n}(t,\cdot)\in W^{1,\infty}(\Omega)^{2}$ and $\left\|\nabla\Phi_{n}(t,\cdot)\right\|_{L^{\infty}(\Omega)^{4}} = \mathrm{Lip}(\Phi_{n}(t,\cdot))$. The Lipschitz continuity implies via Gronwall's lemma 
\begin{equation}
\label{eq:upperbound}
\mathrm{Lip}(\Phi_{n}(t,\cdot))\leq\exp\left(\int\limits_{0}^{t}\mathrm{Lip}(b_{n}(s,\cdot))ds\right).
\end{equation}
The boundedness of $(b_{n})$ in $L^{2}([0,T];H^{3}(\Omega)^{2})$ gives the upper bound of $\eqref{eq:upperbound}$. Hence, the Jacobian determinant $\det\nabla\Phi_{n}(t,\cdot)$ is also uniformly bounded in $L^{\infty}(\Omega)$. According to Lemma $\ref{lem:bv-bounded n}$ this implies that $u^{t}_{n,\epsilon}$ is uniformly bounded in $BV(\Omega)$ w.r.t. $n$. Then, there exists a subsequence $(u^{t}_{n_{k},\epsilon})$ of $(u^{t}_{n,\epsilon})$ such that $u^{t}_{n_{k},\epsilon}$ converges to $u^{t}_{\epsilon}$ in $L^{p}(\Omega)$ (weakly for $p=2$) with $p\leq2$. Considering the integral over time one has 
\begin{eqnarray*}
\lim\limits_{n_{k}\rightarrow\infty}\int\limits_{0}^{T}\left\|u^{t}_{n_{k},\epsilon}-u^{t}_{\epsilon}\right\|_{L^{p}(\Omega)}^{2}dt=\int\limits_{0}^{T}\lim\limits_{n_{k}\rightarrow\infty}\left\|u^{t}_{n_{k},\epsilon}-u^{t}_{\epsilon}\right\|_{L^{p}(\Omega)}^{2}dt\rightarrow0
\end{eqnarray*}
with $p<2$. The exchange of the limit is valid since the integrand is bounded and with the same argument one can derive the weak convergence of $u_{n_{k},\epsilon}$ in $L^{2}([0,T];L^{2}(\Omega))$. \qed
\end{proof}
Now we consider the  minimization problem
\begin{equation}
\label{eq:min-problem}
\inf\limits_{b\in L^{2}([0,T];H^{3,\mathrm{div}}_{0}(\Omega)^{2})}J(b)
\end{equation}
with $J$ according to~\eqref{eq:costfunctional}.
Proving the existence of minimizers  is usually achieved by the direct method \cite{aubert02} and the most difficult part lies in the weak sequential closeness of the solution operator $G$ with respect to $b$.
\begin{theorem}[Weak sequential closeness]
\label{thm:closeness}
Suppose the sequence $(b_{n})\in L^{2}([0,T];H^{3,\mathrm{div}}_{0}(\Omega)^{2})$ is uniformly bounded and converges weakly to $b$ in $L^{2}([0,T];H^{3}(\Omega)^{2})$. Let $u_{n}$ be the corresponding weak solutions of $\eqref{eq:transport}$ with flow field $b_{n}$ and initial value $u_{0}$ (i.e.~$u_n = G(u_0,b)$). Suppose that $u_{n}$ converges to $\hat{u}$ in $L^{2}([0,T];L^{1}(\Omega))$ and $\hat{u}\in L^{2}([0,T];L^{2}(\Omega))$, then $\hat{u}=G(u_{0},b)$. 
\end{theorem}

\begin{proof}
Since $(b_{n})$ converges weakly to $b$ in $L^{2}([0,T];H^{3}(\Omega)^{2})$, it is also valid that
\begin{equation}
\label{eq:weak-conv}
b_{n}\rightharpoonup b\text{ in } L^{2}([0,T];L^{2}(\Omega)^{2}).
\end{equation}
Let us consider the difference $u_u-\hat u$ applying a test function $\varphi\in C^{\infty}_{c}([0,T[\times\Omega)$:
\begin{eqnarray*}
&&\left|\int\limits_{0}^{T}\int\limits_{\Omega}u_{n}(\partial_{t}\varphi+b_{n}\nabla\varphi)-\hat{u}(\partial_{t}\varphi+b\nabla\varphi)dxdt\right|\\
&=&\Bigg|\underbrace{\int\limits_{0}^{T}\int\limits_{\Omega}\partial_{t}\varphi(u_{n}-\hat{u})dxdt}_{(i)}+\underbrace{\int\limits_{0}^{T}\int\limits_{\Omega}\nabla\varphi\cdot(u_{n}b_{n}-\hat{u}b)dxdt}_{(ii)}\Bigg|.
\end{eqnarray*}
Part $(i)$ converges to zero, since $u_{n}\rightarrow\hat{u}$ in $L^{2}([0,T];L^{1}(\Omega))$. Regarding part $(ii)$ we can derive
\begin{eqnarray*}
&&\int\limits_{0}^{T}\int\limits_{\Omega}\nabla\varphi(u_{n}b_{n}-\hat{u}b)dxdt\\
&=&\left(\int\limits_{0}^{T}\int\limits_{\Omega}\nabla\varphi b_{n}(u_{n}-\hat{u})dxdt+\int\limits_{0}^{T}\int\limits_{\Omega}\nabla\varphi\hat{u}(b_{n}-b)dxdt\right)\\
&\leq&\left\|\nabla\varphi\right\|_{L^{\infty}([0,T]\times\Omega)^{2}}\left\|b_{n}\right\|_{L^{2}([0,T];L^{\infty}(\Omega)^{2})}\left\|u_{n}-\hat{u}\right\|_{L^{2}([0,T];L^{1}(\Omega))}\\
&&+\int\limits_{0}^{T}\int\limits_{\Omega}\nabla\varphi\hat{u}(b_{n}-b)dxdt
\end{eqnarray*}
Since $(b_{n})$ is uniformly bounded in $L^{2}([0,T];H^{3}(\Omega)^{2})$, it is also uniformly bounded in $L^{2}([0,T];L^{\infty}(\Omega)^{2})$. Due to the convergence of $u_{n}$ in $L^{2}([0,T];L^{1}(\Omega))$ and $\eqref{eq:weak-conv}$ imply the two summands of last inequality converge respectively to zero.

Since $(u_{n})$ are weak solutions of $\eqref{eq:transport}$, the limit $\hat{u}$ is also a weak solution of $\eqref{eq:transport}$, i.e. $\hat{u}=G(u_{0},b)$.\qed
\end{proof}

\begin{theorem}[Existence of a minimizer]
Suppose $u_{0}\in BV(\Omega)$, then the minimization problem
$\eqref{eq:min-problem}$ has a solution.
\end{theorem}

\begin{proof}
Let $(b_{n})\subset L^{2}([0,T];H^{3,\mathrm{div}}_{0}(\Omega)^{2})$ be a minimizing sequence of the cost functional. The coercivity of $\eqref{eq:costfunctional}$ is a natural property subject to the norm $\eqref{norm:h3}$. From the coercivity one has $(b_{n})$ is uniformly bounded in $L^{2}([0,T];H^{3}(\Omega)^{2})$, then there is a subsequence $(b_{n_{k}})$ of $(b_{n})$ converging weakly to $b$ in $L^{2}([0,T];H^{3}(\Omega)^{2})$. For each $b_{n}$ there exits a unique flow $\Phi_{n}^{-1}$, which is a diffeomorphism in $\Omega$ and absolutely continuous in $[0,T]$. Define 
\begin{equation*}
u_{n,\epsilon}=(u_{0}*\eta_{\epsilon})\circ\Phi^{-1}_{n}.
\end{equation*}
According to Lemma $\ref{lem:uniform boundedness}$ there exists a subsequence $(u_{n_{k},\epsilon})$, which converges to $u_{\epsilon}\in L^{2}([0,T];L^{2}(\Omega))$ in $L^{2}([0,T];L^{1}(\Omega))$ and converges for every $t\in[0,T]$ weakly to $u_{\epsilon}(t)$ in $L^{2}(\Omega)$. Theorem $\ref{thm:closeness}$ implies that $u_{\epsilon}=(u_{0}*\eta_{\epsilon})\circ\Phi^{-1}$.  Hence, it yields that
\begin{center}
\begin{tabular}{ccc}
$\int\limits_{\Omega}u^{t}_{n_{k},\epsilon}\varphi dx$ & $\longrightarrow$ &
$\int\limits_{\Omega}u^{t}_{\epsilon}\varphi dx$
\\
&&\\
$\downarrow$ & & $\downarrow$ \\
&&\\
$\int\limits_{\Omega}u^{t}_{n_{k}}\varphi dx$ & $\longrightarrow$ & $\int\limits_{\Omega}u^{t}\varphi dx$
\end{tabular} 
\end{center}
for every $\varphi\in L^{2}(\Omega)$. The left and right convergences in the diagram are valid due to the property of approximate identities according and then $u^{t}=u_{0}\circ\Phi^{-1}(t,\cdot)$. Hence, $u^{t}_{n_{k}}$ converges weakly to $u^{t}$ in $L^{2}(\Omega)$ for every $t\in[0,T]$. 

The l.s.c. of the first term in $\eqref{eq:costfunctional}$ can be easily derived from $u_{n_{k}}^{T}-u_{T}\rightharpoonup u^{T}-u_{T}$ in $L^{2}(\Omega)$. And the l.s.c. of the second term in $\eqref{eq:costfunctional}$ is valid due to the norm-continuity of $b$.\qed
\end{proof}

%------------------------------------------Optimality condition system and Algorithmus-----------------------------------------------
\section{First-order Optimality Conditions System}
\label{sec:4}
We use the Lagrangian technique to compute the first-order optimality conditions of control problem  \eqref{eq:costfunctional} governed by $\eqref{eq:transport}$ and $\eqref{eq:divergence-free}$.  Let us define first the minimizing functional with Lagrange multipliers $(p,q)$
\begin{equation}
\label{eq:lagrangian}
L(u,b,p,q)=J(u,b)+\int\limits_{0}^{T}\int\limits_{\Omega}(u_{t}+b\cdot\nabla u)pdxdt+\int\limits_{0}^{T}\int\limits_{\Omega}\mathrm{div}bqdxdt,
\end{equation}
the variable $p$ is the adjoint state of $u$ and $q$ is the adjoint state of $b$.
The functional derivatives of \eqref{eq:lagrangian} w.r.t. $u$ and $b$ yield the first-order necessary conditions system
\begin{equation}
\label{eqn:optimality system}
\left\{
\begin{array}{lcl}
u_{t}+b\cdot\nabla u=0,&\quad&u(0)=u_{0}\\[2ex]
p_{t}+b\cdot\nabla p=0,&\quad&p(T)=-(u(T)-u_{T})\\[2ex]
\mathrm{div}b=0,&\quad&\\[2ex]
\lambda\Delta^{3}b+\nabla q=p\nabla u,&\quad& b=0,\nabla_{n} b=0,\\
&&\Delta b=0\text{ on }\partial\Omega.
\end{array}
\right.
\end{equation}

\section{Algorithms}
\label{sec:5}
In this section we will present an efficient numerical algorithm to discretize the optimality conditions system. 
Regarding the forward and backward transport equations in $\eqref{eqn:optimality system}$ one can take 
advantage of explicit formula $\eqref{eq:classic solution}$ and estimate the backward flow 
by the fourth-order Runge-Kutta method. Another possibility for solving the transport equations is to utilize the explicit high-order 
TVD schemes with flux limiter \textquotedblleft superbee\textquotedblright~\cite{hirsch07,kuzmin04,borz02}. It works
very well for preserving the edges of images and avoiding oscillations of solutions.
The last equation of $\eqref{eqn:optimality system}$ is a triharmonic equation which stems from the use of space $H_0^3$ as penalty term in~\eqref{eq:costfunctional}. There are little articles about its numerical schemes, e.g. \cite{quang05}. But the algorithms are either not efficient or difficult to be applied directly. The motivation for this term was that $b$ has to be Lipschitz continuous to obtain a unique flow $\Phi$. If we apply some smooth initial flow $b^{0}$ in the discrete form of $\eqref{eqn:optimality system}$ and replacing $\Delta^3$ with $\Delta$ in~\eqref{eqn:optimality system} still leads to smooth enough $b$. Actually, according to~\cite{diperna89} an initial value $u_0\in L^2(\Omega)$ is transported into an $L^2(\Omega)$-function by a flow field $b\in H^1$. Hence, in our context we can also work with the optimality system
\begin{equation}
\label{eqn:optimality system_var}
\left\{
\begin{array}{lcl}
u_{t}+b\cdot\nabla u=0,&\quad&u(0)=u_{0}\\[2ex]
p_{t}+b\cdot\nabla p=0,&\quad&p(T)=-(u(T)-u_{T})\\[2ex]
\mathrm{div}b=0,&\quad&\\[2ex]
\lambda\Delta b+\nabla q=p\nabla u,&\quad& b=0,\nabla_{n} b=0,\\
&&\Delta b=0\text{ on }\partial\Omega.
\end{array}
\right.
\end{equation}
We remark that the assumption $u_0,u_T\in BV$ is not present in this
model anymore. One could easily use $U=BV$ and the $BV$-norm for the
difference $u(T)-u_T$ since this would only affect the right hand side
of the adjoint equation. However, in this case we have to ensure that
the flow field $b$ is Lipschitz- continuous. In numerical experiments
we found, that this did not alter the results too much and hence, we
use the optimality system~\eqref{eqn:optimality system_var}.

The hierarchical processing according to \cite{barron94_2}, i.e. a coarse to fine calculation, provides a good choice of $b^{0}$. The quality of $b^{0}$ depends strongly on the downsampling and upsampling procedures of images. 

With a divergence free initial value $b^{0}$ we propose a segregation loop in the spirit of \cite{borz02} to interpolate the intermediate image at time $t$:\\
\textbf{Segregation loop I}.\\
Suppose $n=1,\cdots,N_{loop}$ and $N_{loop}$ is the iteration
number. Given $u_{0},u_{T}$, $b^{n-1}(t)$, $\lambda^{n-1}$.  The
iteration process for solving~\eqref{eqn:optimality system_var} at
iteration $n$ proceeds as follows:
\begin{enumerate}
\item Compute $u^{n-1}(t), \nabla u^{n-1}(t)$ and $u^{n-1}(T)$ by the forward transport equation using $u_{0}$ and $b^{n-1}$.
\item Compute $p^{n-1}(t)$ by the backward transport equation using $-(u^{n-1}(T)-u_{T})$ and $b^{n-1}$.
\item Compute $b^{n}(t)$ by the Stokes equation with right-hand side $p^{n-1}(t)\nabla u^{n-1}(t)$ and a $\lambda^{n}$.
\end{enumerate}
After $N_{loop}$ iterations the intermediate image $u^{N_{loop}}(t)$
approximating $u$ at time $t$. Moreover, we use a monotonically
decreasing sequence $(\lambda^{n})$, which converges to a final
$\lambda^{*}$. However, thanks to the theory of Stokes equations
\cite{girault86}, we know that
\begin{equation}
\label{ineq:stokes}
\left\|b(t)\right\|_{H^{1}(\Omega)}\leq\dfrac{C}{\lambda}\left\|p(t)\nabla u(t)\right\|_{H^{-1}(\Omega)}, \text{ a.e. } t\in[0,T].
\end{equation}
In practice we find out that if we choose $(\lambda^{n})$ such that the norm of the right-hand side of $\eqref{ineq:stokes}$ is monotonically increasing, the value of $b(t)$ will be also increasing. However, the final $\lambda^{*}$ cannot be chosen too small such that the minimizing process of $\eqref{eq:costfunctional}$ is ill-posed.

Moreover, since the system $\eqref{eqn:optimality system}$ is a
necessary condition of minimizing functional
$\eqref{eq:costfunctional}$, one expects that the term
$\left\|u(T)-u_{T}\right\|_{L^{2}(\Omega)}$ is not very small. But since this is one of our final goals, we propose a modification of segregation loop I, which poses no requirement for choosing a specific sequence $(\lambda^{n})$ and gives better approximation of intermediate images. We modify segregation loop I as follows:\\
\textbf{Segregation loop II}.\\
Suppose $n=1,\cdots,N_{loop}$ and $N_{loop}$ is the iteration
number. Given $u_{0},u_{T}$, $b^{n-1}(t)$, $\lambda$. The iteration
process at iteration $n$ proceeds as follows:
\begin{enumerate}
\item Compute $u^{n-1}(t), \nabla u^{n-1}(t)$ and $u^{n-1}(T)$by the forward transport equation using $u_{0}$ and $b^{n-1}$.
\item Compute $p^{n-1}(t)$ by the backward transport equation using $-(u^{n-1}(T)-u_{T})$ and $b^{n-1}$.
\item Compute  the solution of the Stokes equations with right-hand side $p^{n-1}(t)\nabla u^{n-1}(t)$ and $\lambda$. Then, denote it by $\delta b^{n-1}(t)$ .
\item $b^{n}(t)=b^{n-1}(t)+\delta b^{n-1}(t)$.
\end{enumerate}
In segregation loop II we utilize the system $\eqref{eqn:optimality system_var}$ to estimate the update of the flow field and update the flow field in step $4$. This point of view is different from the original problem $\eqref{eqn:optimality system_var}$, but interestingly this modification actually solves the necessary condition of another minimizing problem.  If the segregation loop II converges, then the update $\delta b^{n-1}(t)$ converges to zero.  Since the initial value $b^{0}$ is divergence free and in each iteration the update flow $\delta b^{n-1}$ is divergence free, the limit of $b^{n}$  is also divergence free.

We denote $u^{*},p^{*},b^{*},q^{*}$ the limits of particular sequences and in this case $\delta b^{*}=0$.  Setting the limits into $\eqref{eqn:optimality system_var}$ we derive
\begin{equation}
\label{eqn:optimality limits}
\left\{
\begin{array}{lcl}
u^{*}_{t}+b^{*}\cdot\nabla u^{*}=0&\quad&u^{*}(0)=u_{0}\\[2ex]
p^{*}_{t}+b^{*}\cdot\nabla p^{*}=0&\quad&p^{*}(T)=-(u^{*}(T)-u_{T})\\[2ex]
\mathrm{div}b^{*}=0&\quad&b^{*}=0\text{ on }\partial\Omega\\[2ex]
\nabla q^{*}=p^{*}\nabla u^{*}&\quad&
\end{array}
\right.
\end{equation}
Actually, $\eqref{eqn:optimality limits}$ is the optimality system of another constrained minimization problem, namely 
\begin{equation}
\label{eq:short}
\dfrac{1}{2}\left\|u^{*}(T)-u_{T}\right\|^{2}_{L^{2}(\Omega)}
\end{equation}
subject to
\begin{equation}
\label{eqn:shortII}
\left\{
\begin{array}{lcl}
u^{*}_{t}+b^{*}\nabla u^{*} = 0&\quad&u^{*}(0)=u_{0}\\[2ex]
\mathrm{div}b^{*}=0&\quad&b^{*}=0\text{ on }\partial\Omega.
\end{array}
\right.
\end{equation}
Compared to $\eqref{eq:costfunctional}$ the functional
$\eqref{eq:short}$ is not regularized. But if we stop the segregation
loop II on time, i.e. the interpolation error does not vary too much,
then it is not surprising that segregation loop II gives good
approximation results of intermediate images. From the point of view
of regularization theory, one may see the segregation loop II as a
kind of a Landweber method for minimizing
$\|u(T)-u_T\|_{L^2(\Omega)}^2$ which is inspired by a
Tikhonov-functional.

In the most cases the forward interpolation from $u_{0}$ to $u_{T}$ and the backward interpolation from $u_{T}$ to $u_{0}$ are complementary, since the flow is only able to transport objects from somewhere to somewhere, but not able to create some new objects. If in the forward case some new objects appear, then in the backward case the new objects disappear. It means that backward interpolation is more suitable for interpolating the intermediate images. In practice, we take the average of forward and backward interpolations.

\subsection{Hierarchical Method}
In order to get a start value $b^{0}$ for the optimality system, the hierarchical processing is a good ansatz. It can be understood in level $l$ in the following steps:
\begin{enumerate}
\item Downsample the images into level $l$.
\item Solve system $\eqref{eqn:optimality system_var}$ in level $l$ out and get $b^{l}$.
\item Upsample the optical flow into level $l-1$ and get $b^{l-1}$.
\end{enumerate}
The estimated optical flow $b^{l-1}$ is a start value of the hierarchical method in level $l-1$. In coarsest level we assume the start value is zero. As above mentioned, the down- and up-sampling methods are decisively, i.e. it is supposed to lose the local structures of objects as small as possible while down- and up-sampling the images or the optical flow.

In practice, we apply bicubic interpolation \cite{william07} for the sampling, since it has fewer interpolation artifacts than bilinear interpolation or nearest-neighbor interpolation. Compared to the Gaussian pyramid \cite{peter83} the downsampled images by bicubic interpolation does look not so blurred.

\subsection{Numerical Schemes for Transport Equations}
To discretize the transport equations we can use the second-order TVD scheme. It is also suitable for the backward transport equation, since we can reform it into the forward problem by setting $t':=T-t$:
\begin{equation*}
p_{t'}-b\cdot\nabla p=0,\quad p(0)=-(u(0)-u_{T}).
\end{equation*}
Suppose the image size is $N\times M$, $h$ and $\Delta t$ are the mesh sizes in space and time, respectively with mesh index $i=1,\cdots,N, j=1,\cdots,M$ in space and $k=1,\cdots,K$ in time. The stability condition of the scheme, usually called CFL condition \cite{aubert02}, is 
\begin{equation*}
\sigma_{CFL}:=\max(|v|_{\max},|w|_{\max})\frac{\Delta t}{h}\leq1.
\end{equation*}
by setting $b:=(v,w)$. In practice we choose $\Delta t$ such that  $\sigma_{CFL}=0.1$. The TVD scheme of the forward transport equation is:
\begin{eqnarray*}
u_{t}|^{k}_{ij}&=&\frac{u_{ij}^{k+1}-u_{ij}^{k}}{\Delta t},\\
-vu_{x}|_{ij}^{k}&=&\frac{v_{ij}^{+}}{h}\left[
1+\frac{1}{2}\chi(r^{+}_{i-\frac{1}{2},j})-\frac{1}{2}\frac{\chi(r^{+}_{i-\frac{3}{2},j})}{r^{+}_{i-\frac{3}{2},j}}
\right](u_{i-1,j}^{k}-u_{ij}^{k})\\
&&-\frac{v_{ij}^{-}}{h}
\left[
1+\frac{1}{2}\chi(r^{-}_{i+\frac{1}{2},j})-\frac{1}{2}\frac{\chi(r^{-}_{i+\frac{3}{2},j})}{r^{-}_{i+\frac{3}{2},j}}
\right]\\
&&\cdot(u_{i+1,j}^{k}-u_{ij}^{k}),
\end{eqnarray*}
where $v_{ij}^{+}=\max(v_{ij},0),v_{ij}^{-}=\min(v_{ij},0)$ and the flux difference ratios are defined as
\begin{eqnarray*}
r^{+}_{i-\frac{1}{2},j}=\frac{u_{i+1,j}^{k}-u_{ij}^{k}}{u_{ij}^{k}-u_{i-1,j}^{k}},&& 
r^{+}_{i-\frac{3}{2},j}=\frac{u_{ij}^{k}-u_{i-1,j}^{k}}{u_{i-1,j}^{k}-u_{i-2,j}^{k}},\\
r^{-}_{i+\frac{1}{2},j}=\frac{u_{ij}^{k}-u_{i-1,j}^{k}}{u_{i+1,j}^{k}-u_{ij}^{k}},&& 
r^{-}_{i+\frac{3}{2},j}=\frac{u_{i+1,j}^{k}-u_{ij}^{k}}{u_{i+2,j}^{k}-u_{i+1,j}^{k}}.
\end{eqnarray*}
In the similar way we can discretize the term $-wu_{y}$. The superbee limiter function is given by
\begin{equation*}
\chi(r)=\max(0,\min(2r,1),\min(r,2)). 
\end{equation*}
To compute the spatial derivatives of images we use the standard three-point formula:
\begin{eqnarray*}
pu_{x}|_{ij}&=&\frac{1}{2h}(-u_{i-1,j}+u_{i+1,j})p_{ij},\\
pu_{y}|_{ij}&=&\frac{1}{2h}(-u_{i,j-1}+u_{i,j+1})p_{ij}.
\end{eqnarray*}
Another way for solving the transport equation is to utilize the characteristic solution. From $\eqref{eq:classic solution}$ we know the keypoint is to solve the backward flow starting from $(t,x)$
\begin{equation}
\label{eqn:ode-back}
\left\{
\begin{array}{l}
\dfrac{\partial \Phi}{\partial s} = b(s,\Phi)\quad\text{in }[0,t[\times\Omega,\\
\\
\Phi(t,x) = x\quad\text{in }\Omega.
\end{array}
\right.
\end{equation}
To solve $\eqref{eqn:ode-back}$ numerical efficiently we use Runge-Kutta 4th order method \cite{william07}. We discretize $[t,0]$ with time step $\Delta t=0.1$ and utilize a constant flow $b$ over $[t,0]$ due to saving the memory and computational cost. In this scheme we have to interpolate the flow $b(t,x)$ with some non-integer $x$, since only the flow $b(t,\cdot)$ with integer coordinates is given. For this we use bilinear interpolation (a bicubic interpolation leads to almost the same results).
Then, we warp the image $u_{0}$ with the coordinates calculated by $\eqref{eqn:ode-back}$ using cubic spline predefined in Matlab to approximate $u(t,x)$.

\subsection{Finite Element Methods for Stokes Equations}
As previously mentioned, after replacing $\Delta^{3}$ with $\Delta$ it is immediately seen that the last two equations in $\eqref{eqn:optimality system_var}$ are the Stokes equations. Stokes flow estimation was investigated in \cite{ruhnau06} and Suter applied the mixed finite element method \cite{suter94} for solving it. Moreover, the approximation of velocity field $b(t,\cdot)$ and pressure $q(t,\cdot)$ will achieved by the polynomial of second order (P2) and first order (P1), so-called Taylor and Hood elements \cite{elman91}. If the chosen finite element spaces satisfy the inf-sup condition, also called LBB condition \cite{elman91,fortin91}, then the method is stable.

The variational problem of the Stokes equations reads as follows:
\begin{equation}
\label{eqn:stockes-weak}
\left\{
\begin{array}{rcl}
a(b(t),v)+c(v,q(t)) &=& (f(t),v)\quad\forall v\in V,\\[2ex]
c(b(t),w) & = & 0,\quad\forall w\in W\,
\end{array}
\right.
\end{equation}
and the bilinear forms are defined by
\begin{eqnarray*}
a(b(t),v)&=&\int\limits_{\Omega}\lambda\nabla b(t)\nabla vdxdy,\\
c(v,q(t))&=&\int\limits_{\Omega}(\mathrm{div}v)q(t)dxdy,\\
(f(t),v)&=&-\int\limits_{\Omega}f(t)vdxdy,
\end{eqnarray*}
where $f:=p\nabla u,V:=H^{1}_{0}(\Omega)^{2}$ and 
\begin{equation*}
W:=\left\{w\in L^{2}(\Omega)~\Big|~\int\limits_{\Omega}wdxdy=0\right\}.
\end{equation*} 
The discretization of $\eqref{eqn:stockes-weak}$ using the mixed finite element produces a linear system of the form
\begin{equation}
\label{eq:matrix}
\begin{pmatrix}
A & C^{t} \\
C & 0 
\end{pmatrix}
\dbinom{b_{MN}}{p_{Q}}
=
\dbinom{f_{MN}}{0}.
\end{equation}
The approximation coefficients $b_{MN},p_{Q}$ and $f_{MN}$ are w.r.t. the basis of finite element spaces $V_{h}$ and $W_{h}$. The stiffness matrix $A$ has the following block form:
\begin{equation*}
A = 
\begin{pmatrix}
A_{1} & 0 \\
0 & A_{1}
\end{pmatrix},
\end{equation*}
where $A_{1}=\left(\int_{\Omega}\nabla\varphi_{i}\nabla\varphi_{j}dxdy\right)_{ij},i,j=1,\cdots,MN$ and $\varphi_{i}$ are the basic functions of $V_{h}$. The matrix $C^{t}$ has also a block form
\begin{equation*}
C^{t}=\dbinom{C_{1}^{t}}{C_{2}^{t}},
\end{equation*}
\begin{eqnarray*}
C_{1}^{t}&=&\left\{\int\limits_{\Omega}\frac{\partial\varphi_{i}}{\partial x}\psi_{j}dxdy ~\Big|~ i=1,\cdots,MN;j=1,\cdots,Q\right\}\\ 
C_{2}^{t}&=&\left\{\int\limits_{\Omega}\frac{\partial\varphi_{i}}{\partial y}\psi_{j}dxdy ~\Big|~ i=1,\cdots,MN;j=1,\cdots,Q\right\}.
\end{eqnarray*}
Similarly, $\psi_{i}$ are the basic functions of $W_{h}$. The vector $f=(f_{1},f_{2})^{t}$ is composed of scalar products $(f_{1},\varphi_{i})$ and $(f_{2},\varphi_{i})$ for $i=1,\cdots,MN$. We derive the interpolation polynomial of $f_{1},f_{2}$ w.r.t. the basic functions
\begin{eqnarray*}
f_{1}^{h}&=&\sum\limits_{i=1}^{MN}f_{1}(x_{i})\varphi_{i}\\
f_{2}^{h}&=&\sum\limits_{i=1}^{MN}f_{2}(x_{i})\varphi_{i},
\end{eqnarray*}
where $x_{i}$ is the corresponding measurement point of $\varphi_{i}$. Then,
\begin{eqnarray*}
f_{i}=(f_{1}^{h},\varphi_{i})&=&\sum\limits_{j=1}^{MN}f_{1}(x_{j})\int\limits_{\Omega}\varphi_{j}\varphi_{i}dxdy,~i=1,\cdots,MN \\
f_{i}=(f_{2}^{h},\varphi_{i})&=&\sum\limits_{j=1}^{MN}f_{2}(x_{j})\int\limits_{\Omega}\varphi_{j}\varphi_{i}dxdy,~i=MN+1,\cdots,2MN.
\end{eqnarray*}
For simplifying the estimation we just need to define the basic functions of a single element, i.e. a triangle or square, and derive the corresponding element stiffness matrix and element mass matrix, then assemble them into $A_{1}$, $C_{1}$, $C_{2}$, $f_{MN}$.

Since the matrix in $\eqref{eq:matrix}$ is sparse and symmetric, but not positive definite, the system $\eqref{eq:matrix}$ can be numerically solved by the routine bicgstab predefined in MATLAB.

\section{Numerical Experiments}
\label{sec:6}

\subsection{Parameter Choice Rule}
The essential parameters of the quality of image interpolation are the regularization parameter $\lambda$ and the downsampling level $l$. Experimentally, we find out that the optimal regularization parameter $\lambda_{opt}$ and $l$ are coupled. The downsampling level should be so adapted that at the lowest level $L$ the estimated optical flow is accurate with a $\lambda_{opt}^{L}$. At the higher level $l$ with $l<N$ the parameter $\lambda_{opt}^{l}$ is larger than $\lambda_{opt}^{N}$. In practice, we choose $\lambda_{opt}^{l}$ with $l<N$ by the following strategy:
\begin{enumerate}
\item Find a pair $(\lambda_{opt}^{L},L)$ experimentally at the lowest level $L$.
\item Choose $\lambda_{opt}^{l-1}$ such that $\lambda_{opt}^{l-1}/\lambda_{opt}^{l} \in [10^{0.2} 10^{0.5}]$ and the interpolation errors decrease at level $l-1$.
\end{enumerate}
The difference between segregation loop I and II lies in that segregation loop II equips with a constant $\lambda^{l}_{opt}$ at each level and segregation loop I applies a monotonically decreasing sequence converging to $\lambda^{l}_{opt}$ at each level. In case the image size is around $600\times400$ we set the lowest level $L=3$ and $\lambda^{L}_{opt}\in[10^{5} 10^{5.5}]$.

\subsection{Numerical Results}
To illustrate the effect of our intermediate interpolated images, we apply the interpolation error (IE) introduced by \cite{BakerSLRBS07}. Moreover, the IE measures the root-mean-square (RMS) difference between the ground-truth image $\tilde{u}$ and the interpolated image $u$
\begin{equation*}
IE = \left(\dfrac{1}{MN}\sum\limits_{i=1}^{N}\sum\limits_{j=1}^{M}(u(x_{i},y_{j})-\tilde{u}(x_{i},y_{j}))^{2}\right)^{\frac{1}{2}},
\end{equation*}
where $M\times N$ is the image size. We test our methods  on  the datasets generated by Middlebury with public ground-truth interpolation:
\begin{itemize}
\item Dimetrodon with size $584\times388$
\item Venus with size $420\times380$
\end{itemize}
Every dataset is composed of three images and the mid-image is the ground-truth interpolation at time $0.5$ if we assume the evolution process of three images lasts time $T=1$. To evaluate the interpolation we can compare our interpolation results with the ground-truth by means of IE measure.The ranking of the interpolation results calculated by segregation loop I and II refers to Table $\ref{tab:1}$.  As in \cite{BakerSLRBS07} mentioned the Pyramid LK method and Mediaplayer$^{\mathrm{TM}}$ are significantly better for interpolation than for ground-truth motion, since e.g. Mediaplayer$^{\mathrm{TM}}$ tends to overly extend the flow into textureless regions, which are not significantly affected by image interpolation.  According to Table $\ref{tab:1}$ segregation loop II works better than some classic methods and more accurate than segregation loop I. The places where the interpolation errors take place refer to Fig. $\ref{fig:1}-\ref{fig:2}$. As a result our methods especially segregation loop II work efficiently in image interpolation.
\begin{table}[!htbp]
\begin{tabular}{ccc}
 &  Dimetrodon & Venus\\
\textbf{Segregation loop I} & $\textbf{2.25}$ & $\textbf{6.67}$\\
\textbf{Segregation loop II} & $\textbf{1.95}$ & $\textbf{3.63}$\\
Stich et al. & $1.78$ & $2.88$ \\
Pyramid LK & $2.49$ & $3.67$\\
Bruhn et al. & $2.59$ & $3.73$\\
Black and Anandan & $2.56$ & $3.93$\\
Mediaplayer$^{\mathrm{TM}}$ & $2.68$ & $4.54$ \\
Zitnick et al. & $3.06$ & $5.33$\\
\end{tabular}
\caption{Interpolation errors calculated by our methods using the Middlebury datasets by comparison to the ground truth interpolation with 
results taken from \cite{stich08}.}
\label{tab:1}
\end{table}

The whole interpolation process of Middlebury datasets is accomplished by 9 generated images respectively using segregation loop I and II. The additional data generated into films are given in Online Resource.

\begin{figure*}[!htbp]
\centering
\begin{tabular}{cc}
\subfigcapskip = 0.2cm
\subfigure[]{\includegraphics[width=0.45\textwidth]{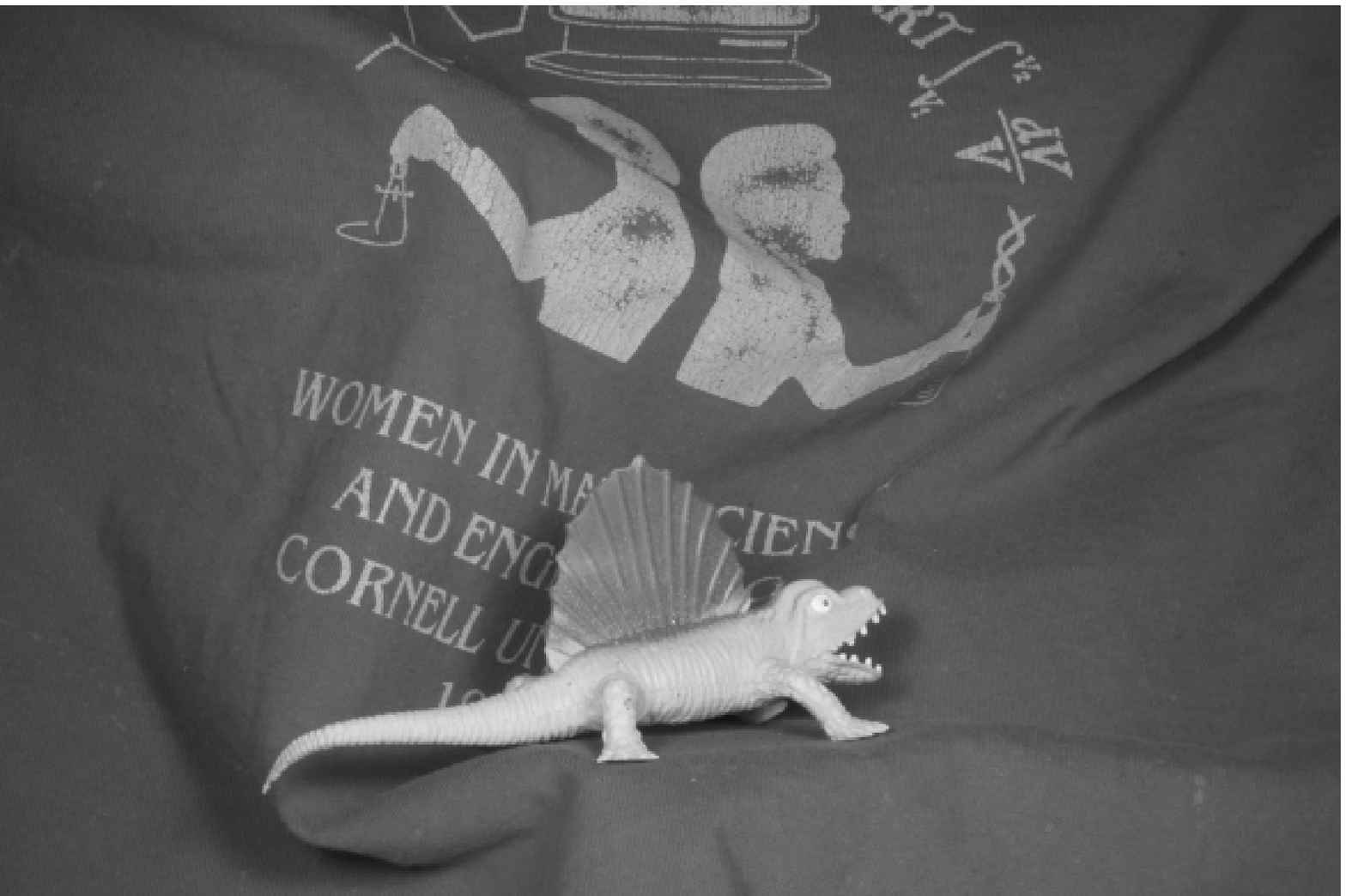}}
&
\subfigcapskip = 0.2cm
\subfigure[]{\includegraphics[width=0.45\textwidth]{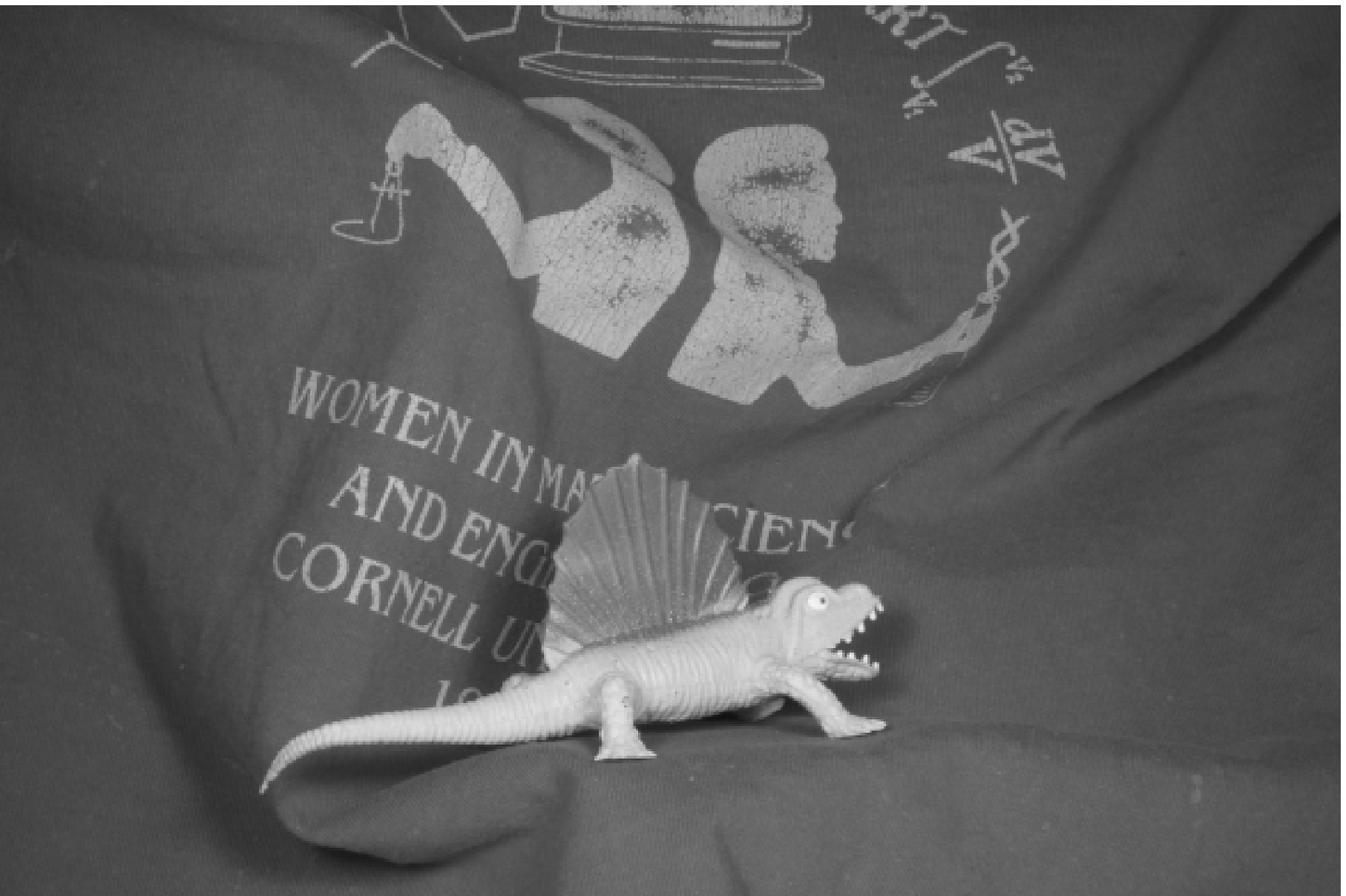}} \\

\subfigcapskip = 0.2cm
\subfigure[]{\includegraphics[width=0.45\textwidth]{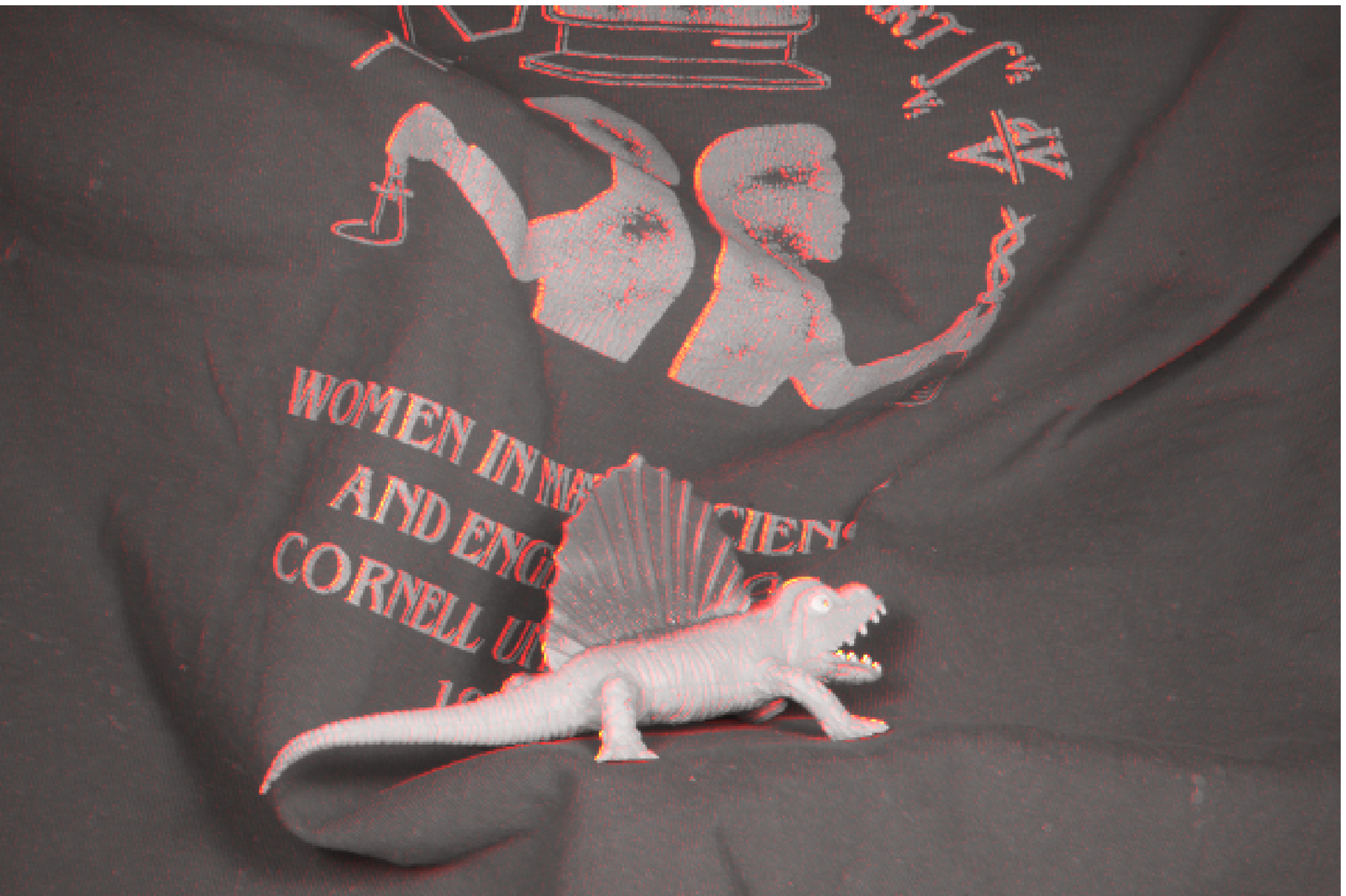}} 
&
\subfigcapskip = 0.2cm
\subfigure[]{\includegraphics[width=0.45\textwidth]{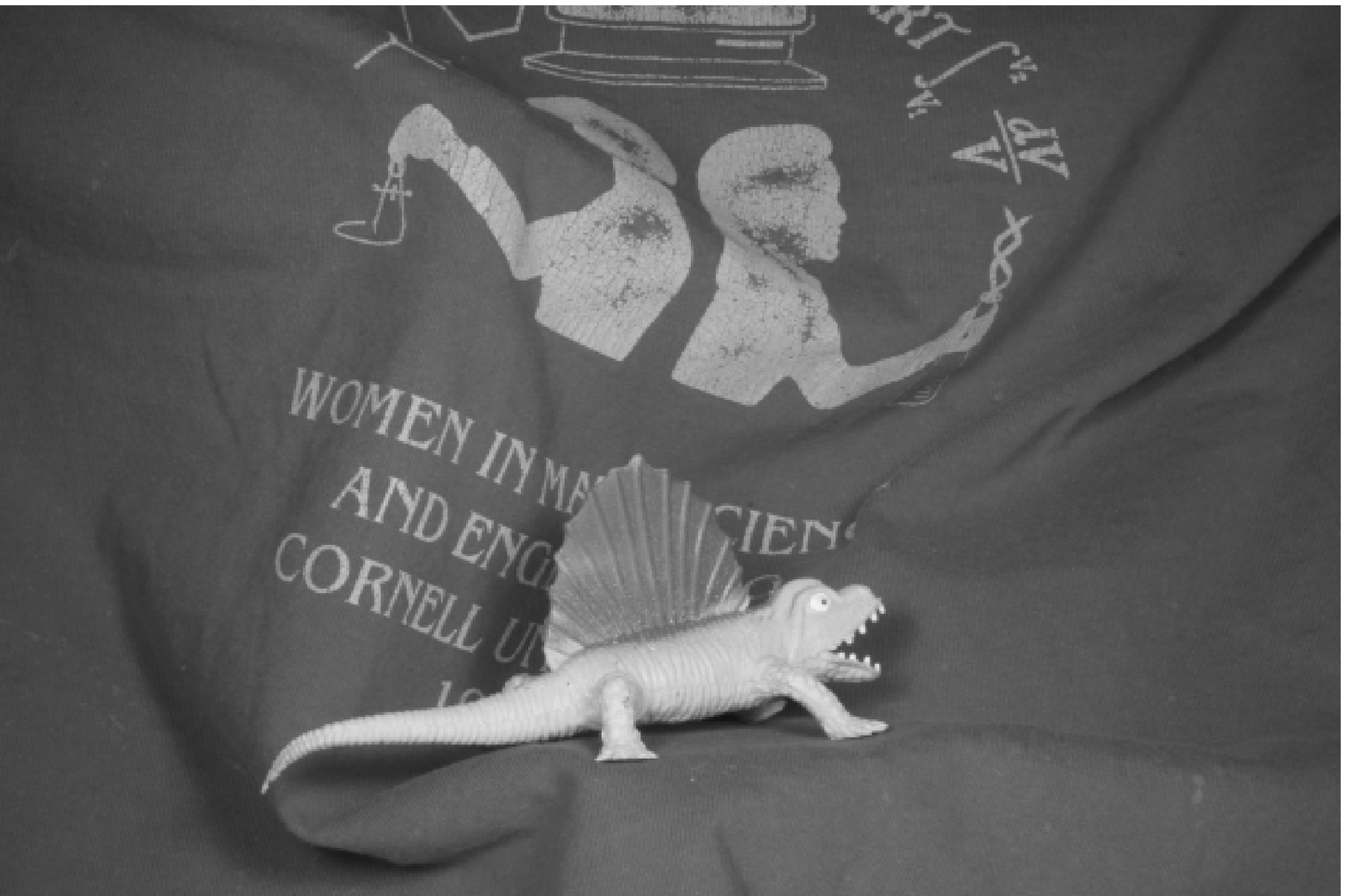}} \\

\subfigcapskip = 0.2cm
\subfigure[]{\includegraphics[width=0.45\textwidth]{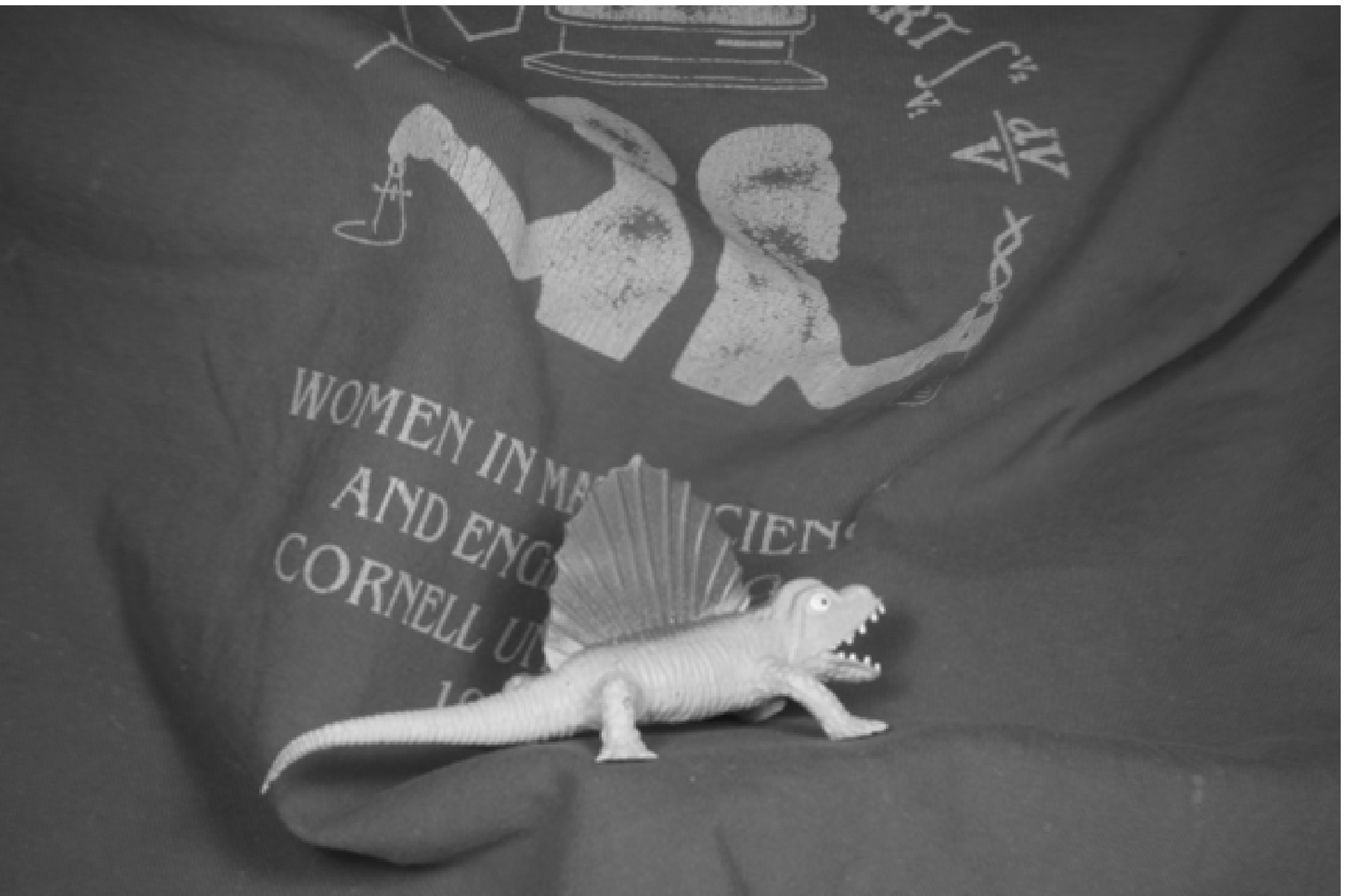}}
&
\subfigcapskip = 0.2cm
\subfigure[]{\includegraphics[width=0.45\textwidth]{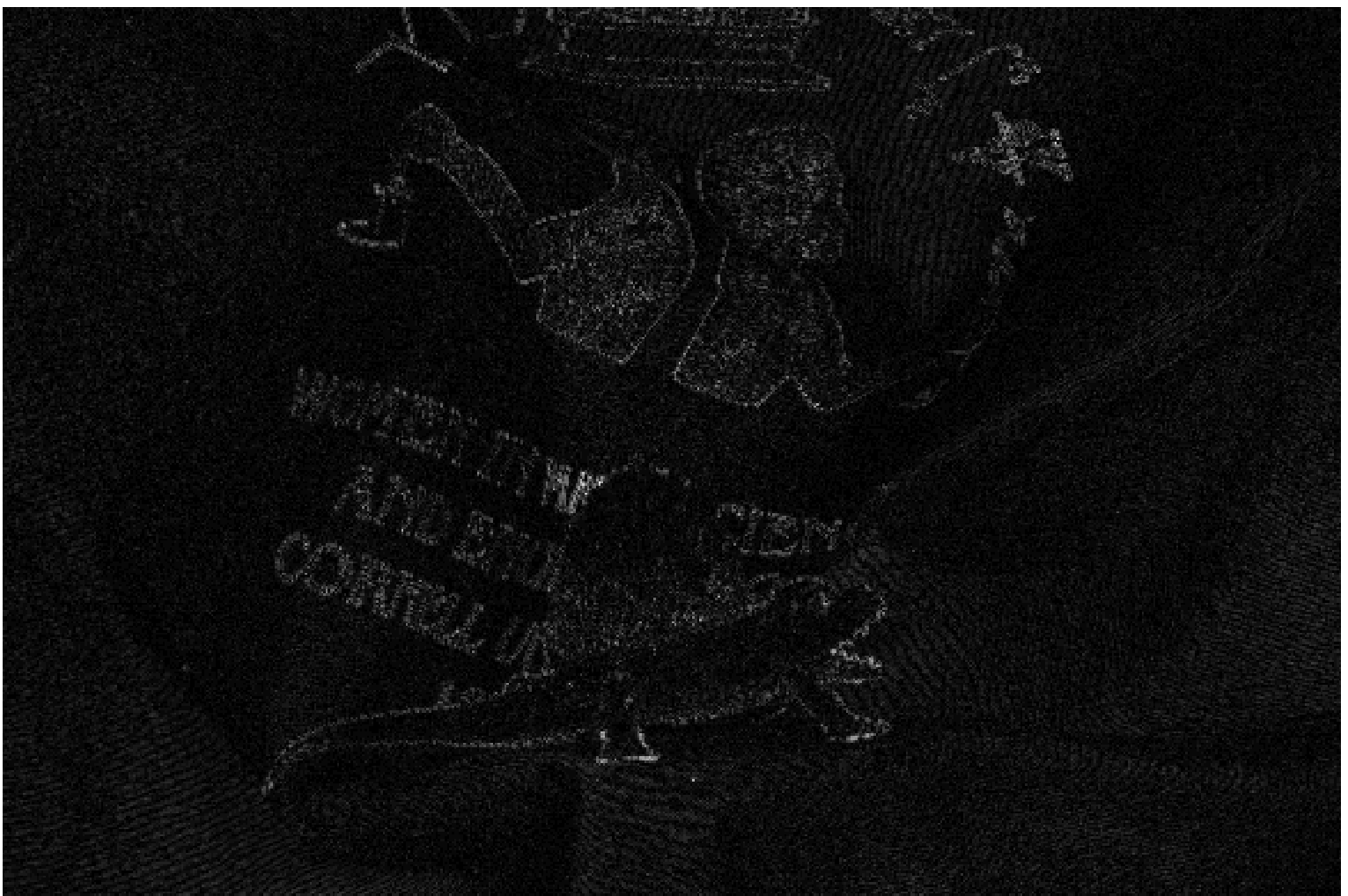}} \\

\subfigcapskip = 0.2cm
\subfigure[]{\includegraphics[width=0.45\textwidth]{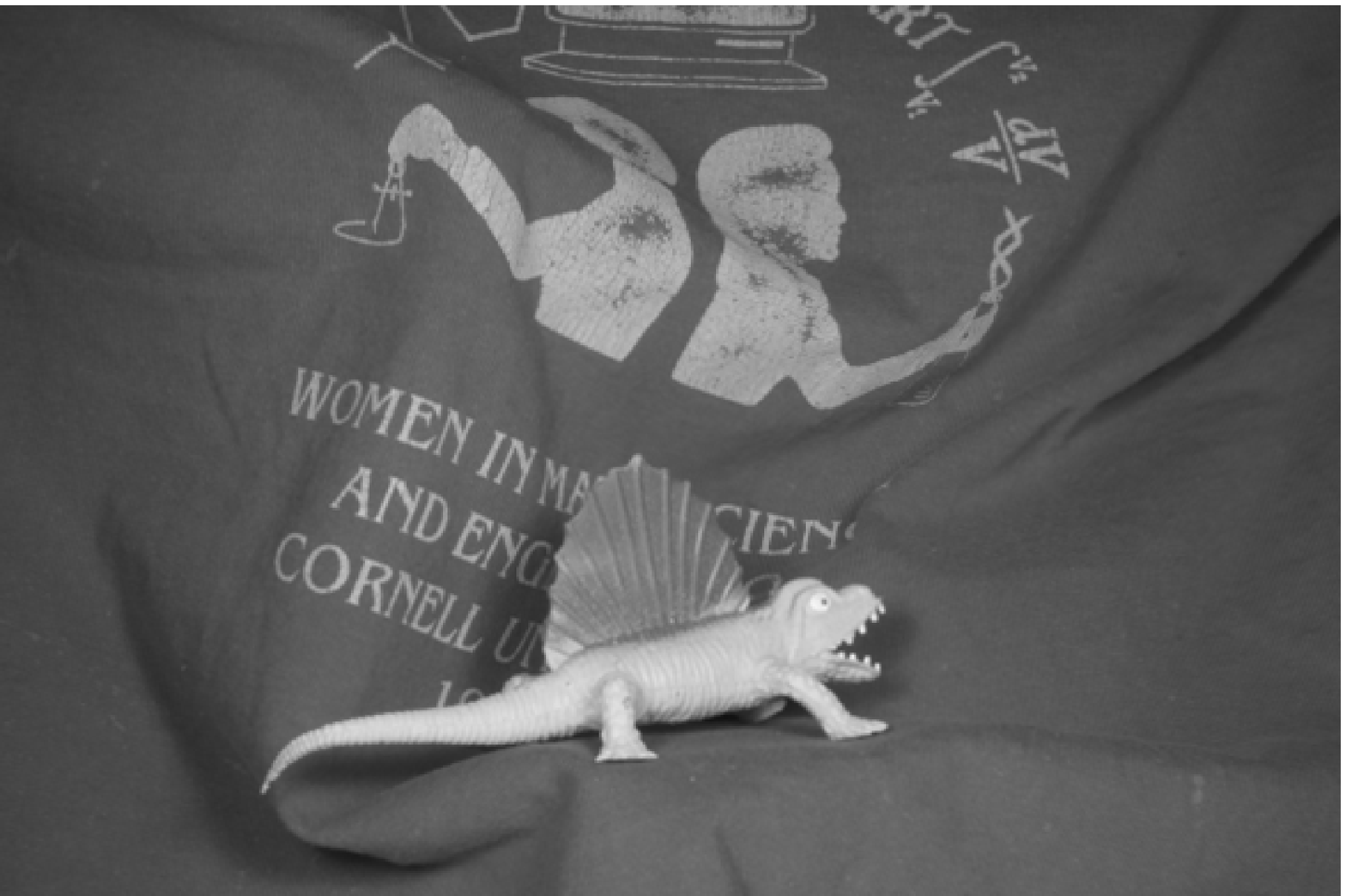}}
&
\subfigcapskip = 0.2cm
\subfigure[]{\includegraphics[width=0.45\textwidth]{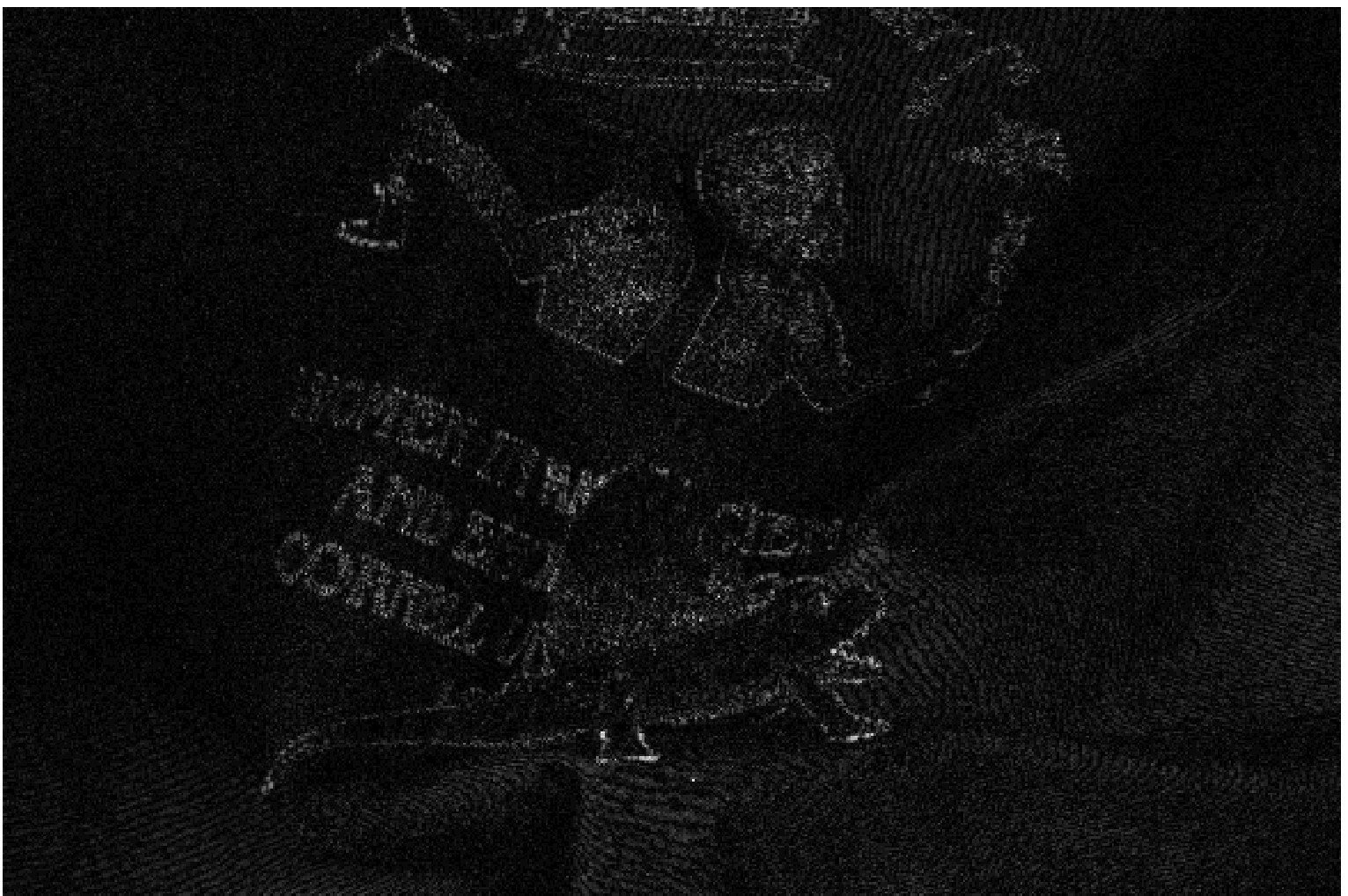}}
\end{tabular}
\caption{(a) $u_{0}$. (b) $u_{T}$. (c) $u_{0}$ plus the colored difference between $u_{0}$
 and $u_{T}$. (d) The groundtruth interpolation at time $T/2$  from the Middlebury datasets. (e)  The generated interpolation at time $T/2$ by segregation loop I.
(f)  The absolute difference between (d) and (e). (g) The generated interpolation at time $T/2$ by segregation loop II.
(h) The absolute difference between (d) and (g).}
\label{fig:1}       % Give a unique label
\end{figure*}

% Venus
\begin{figure*}[!htbp]
\centering
\begin{tabular}{cc}
\subfigcapskip = 0.2cm	
\subfigure[]{\includegraphics[width=0.34\textwidth]{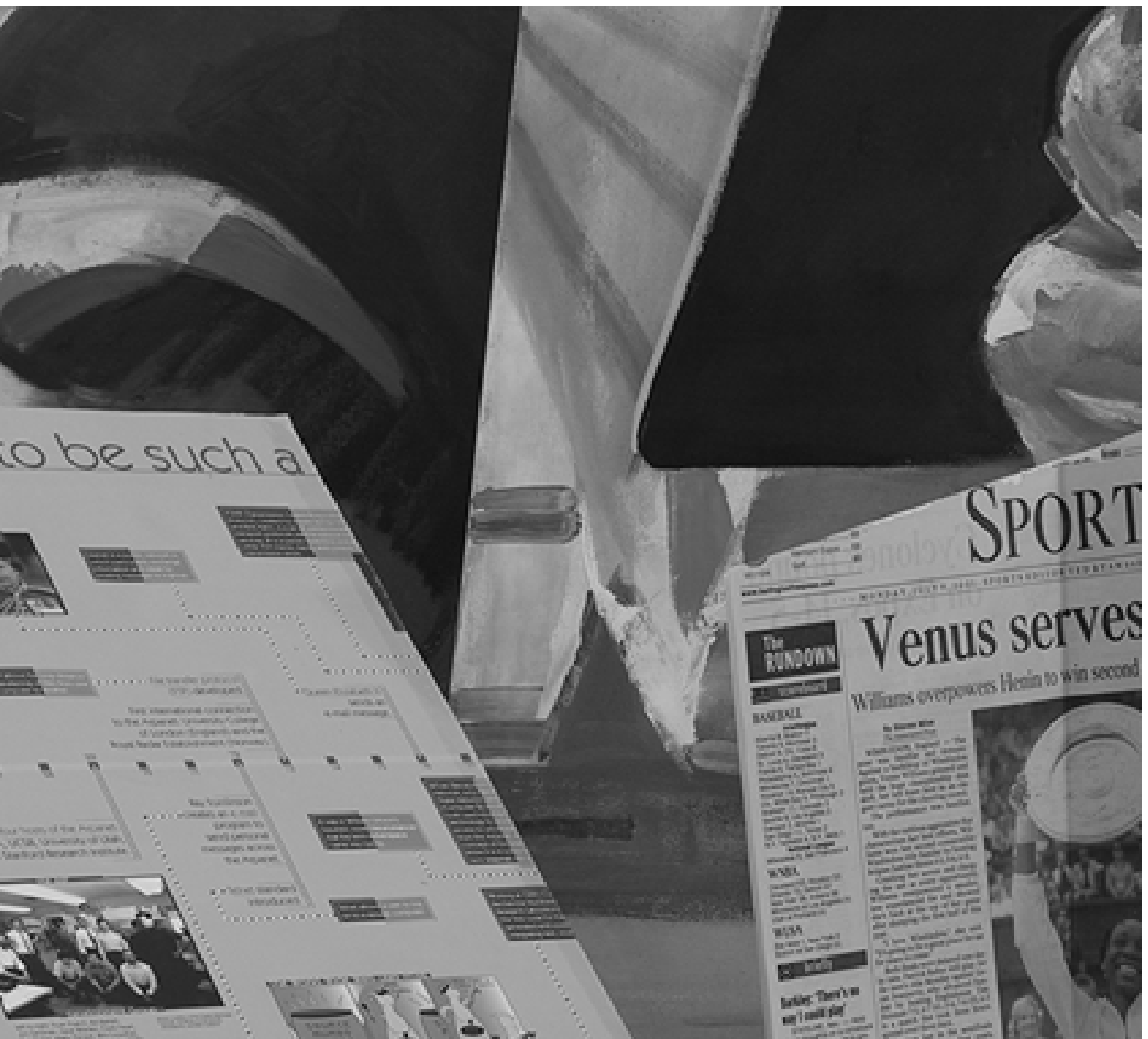}}
&
\subfigcapskip = 0.2cm
\subfigure[]{\includegraphics[width=0.34\textwidth]{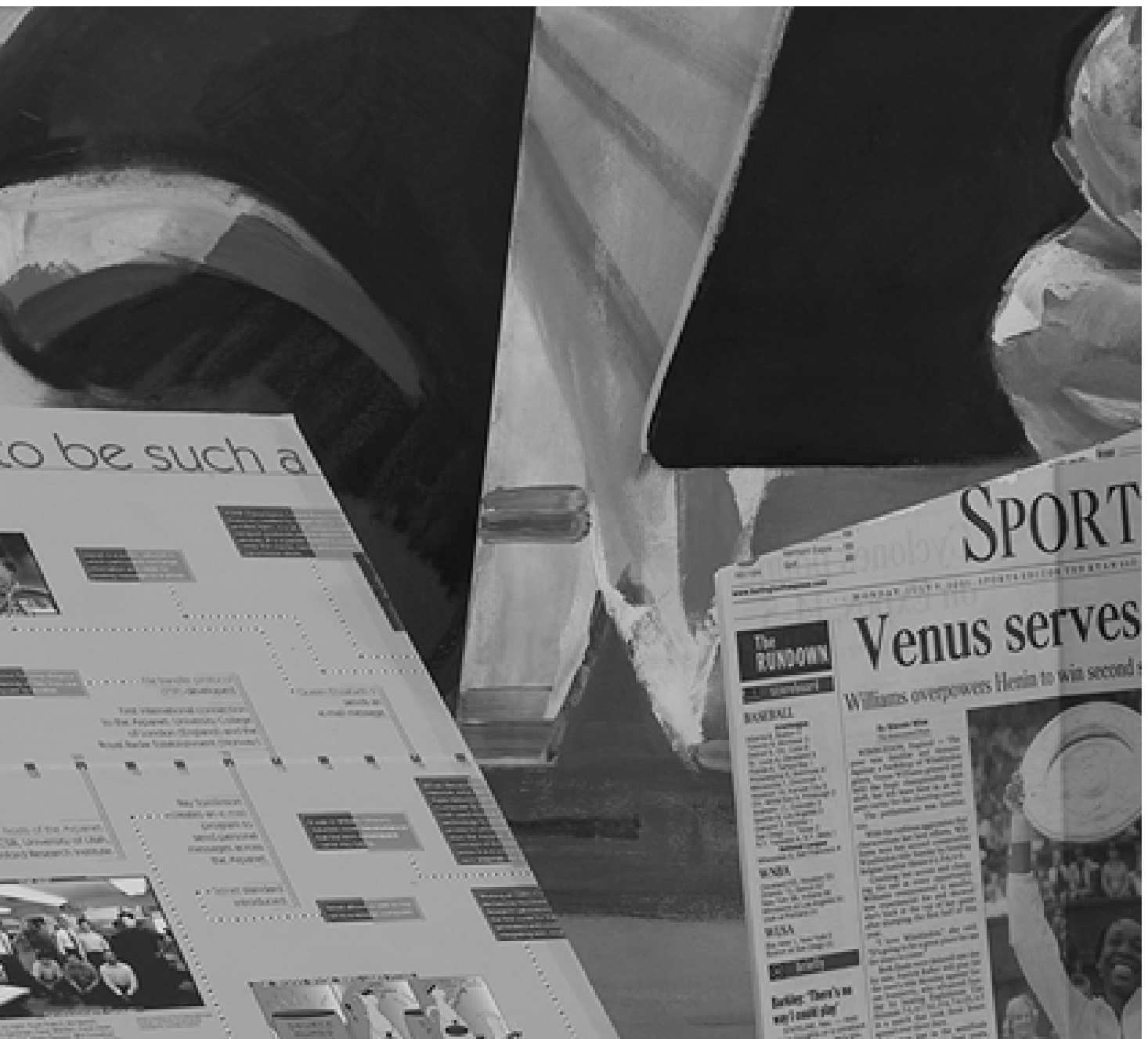}} \\

\subfigcapskip = 0.2cm
\subfigure[]{\includegraphics[width=0.34\textwidth]{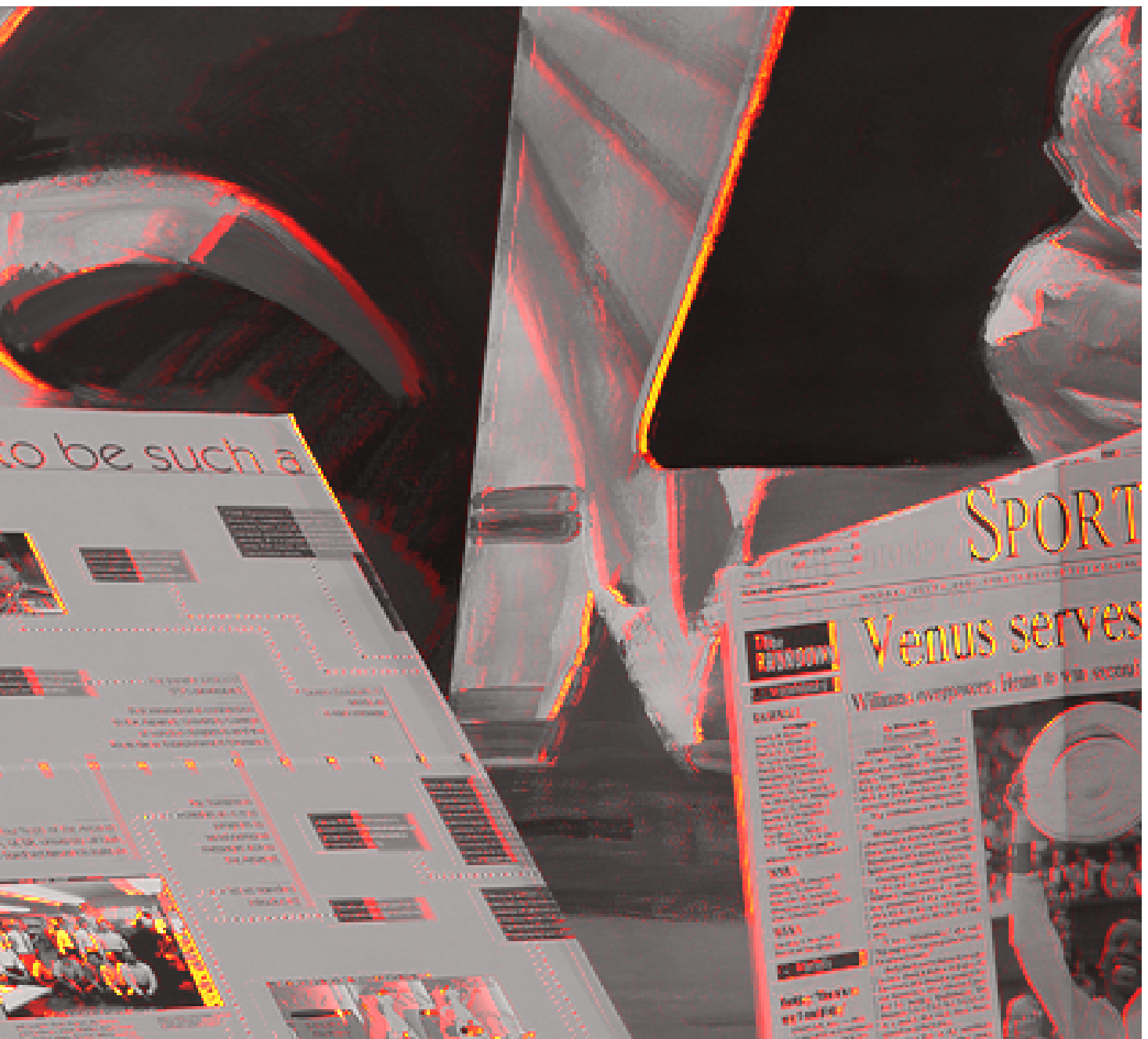}} 
&
\subfigcapskip = 0.2cm	
\subfigure[]{\includegraphics[width=0.34\textwidth]{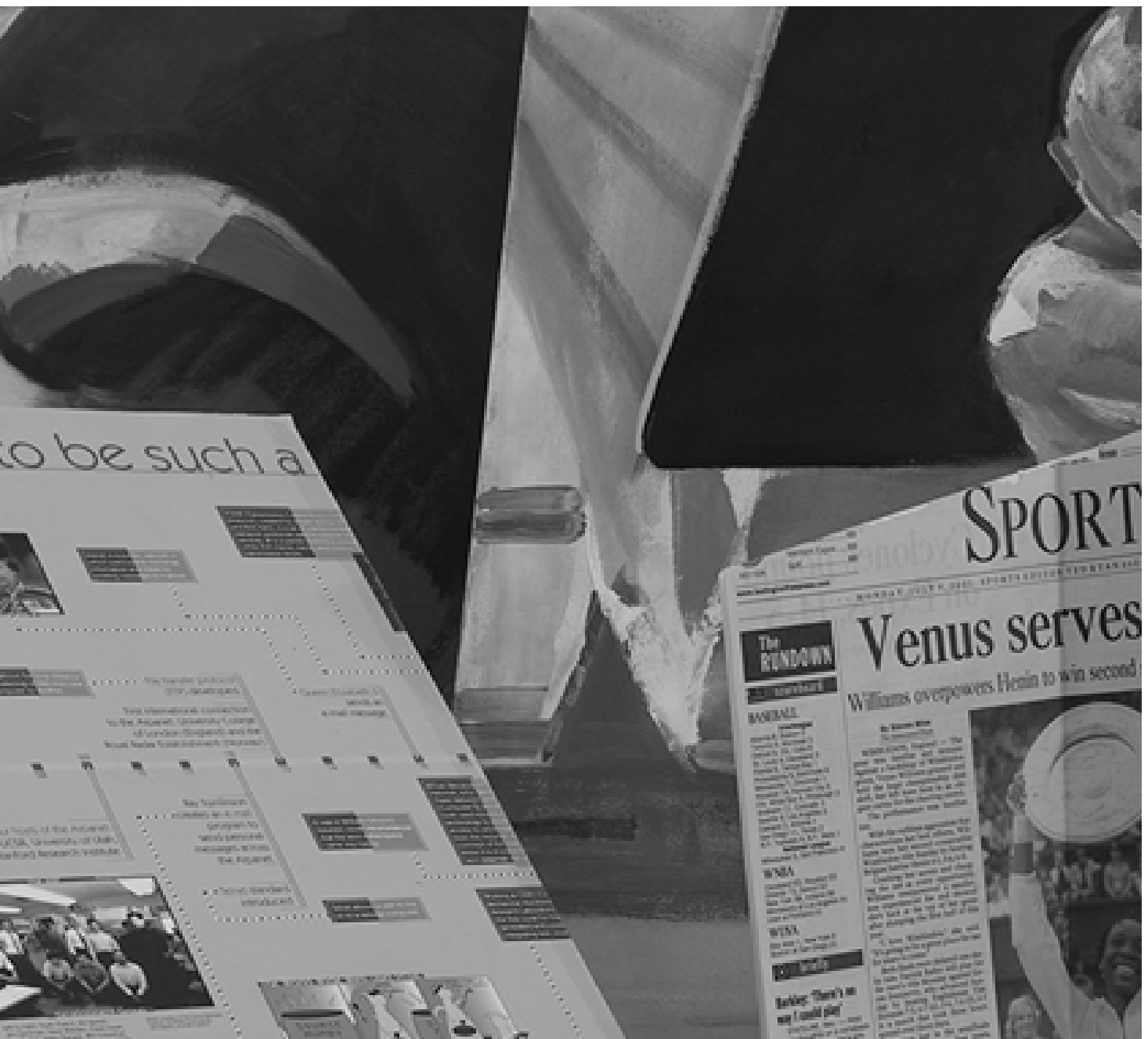}} \\

\subfigcapskip = 0.2cm
\subfigure[]{\includegraphics[width=0.34\textwidth]{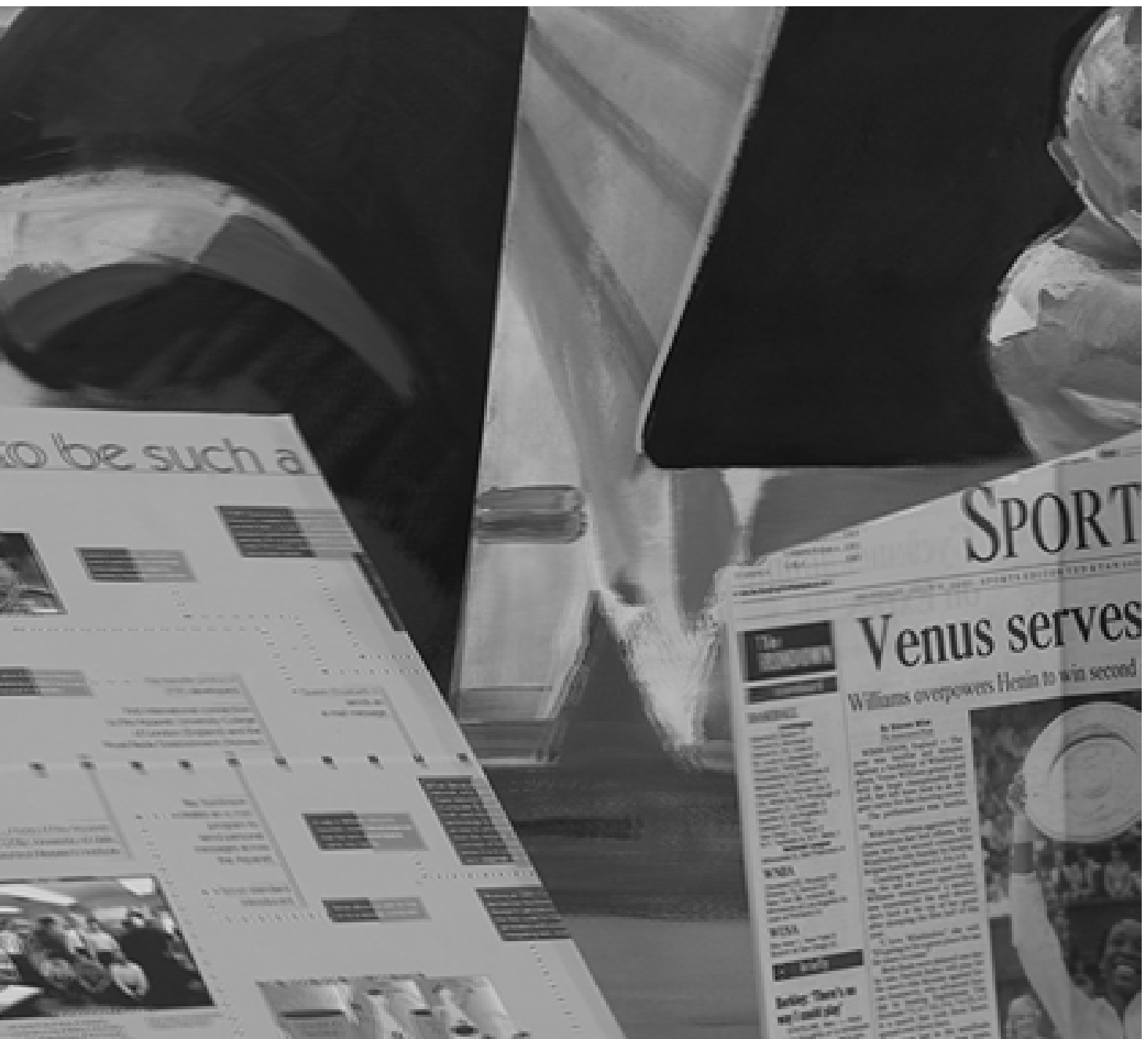}}
&
\subfigcapskip = 0.2cm
\subfigure[]{\includegraphics[width=0.34\textwidth]{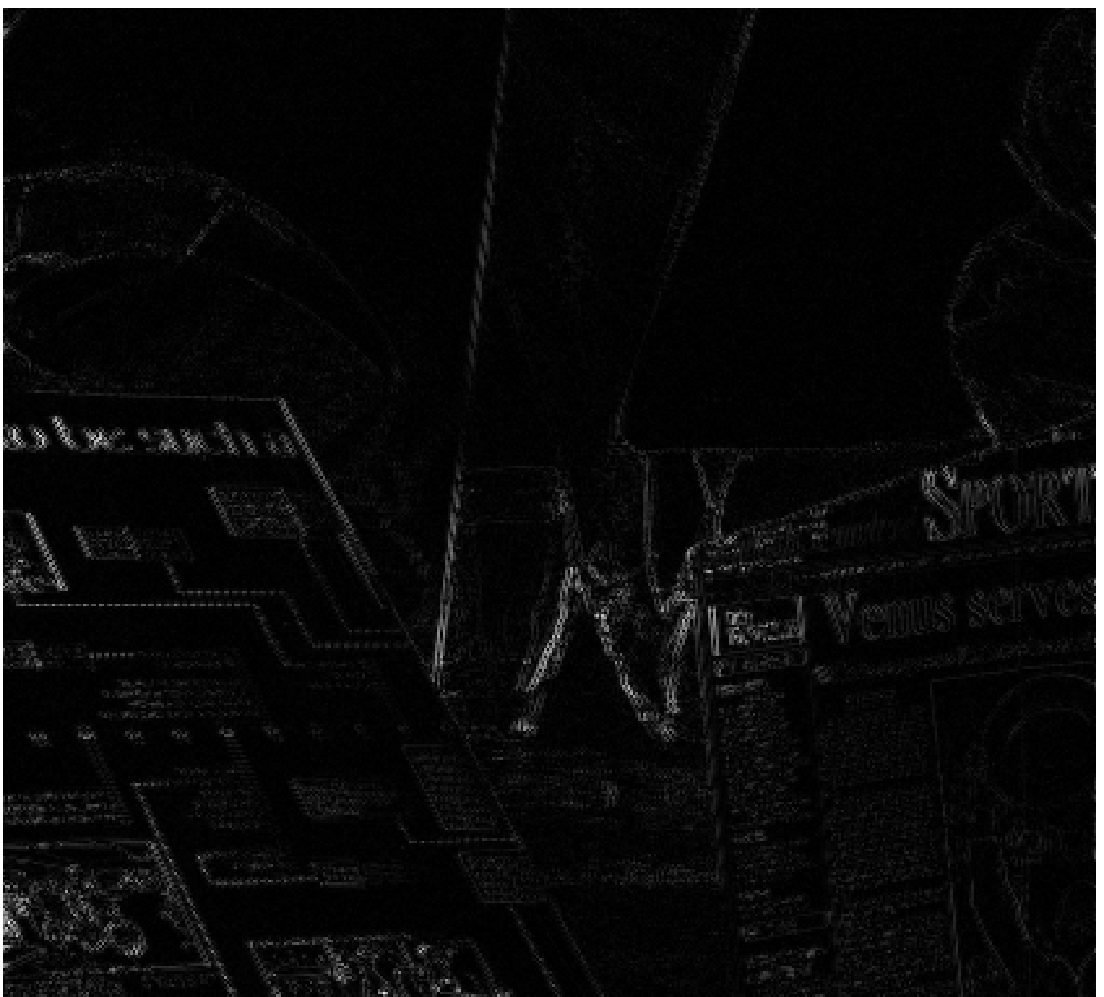}} \\

\subfigcapskip = 0.2cm
\subfigure[]{\includegraphics[width=0.34\textwidth]{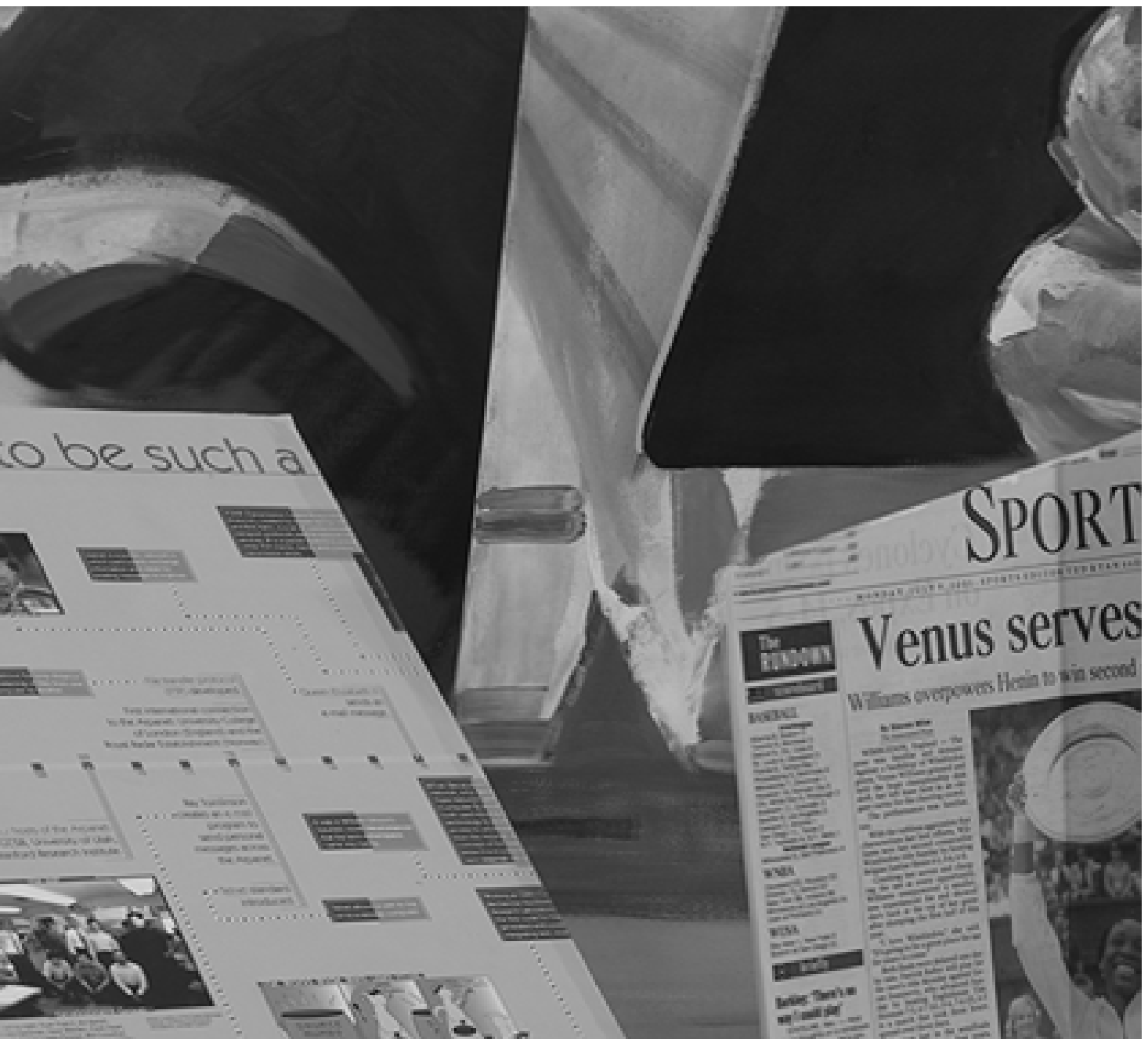}}
&
\subfigcapskip = 0.2cm
\subfigure[]{\includegraphics[width=0.34\textwidth]{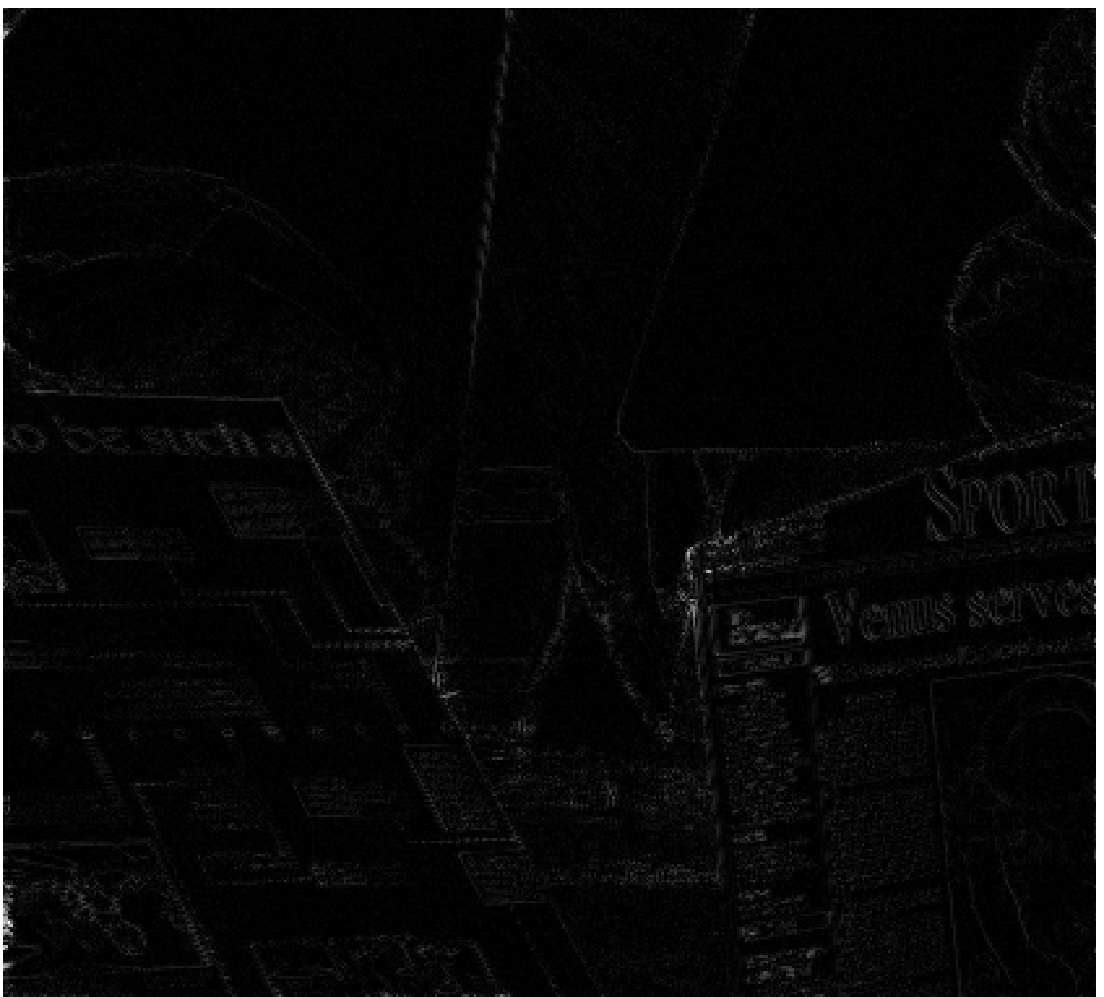}}
\end{tabular}
\caption{(a) $u_{0}$. (b) $u_{T}$. (c) $u_{0}$ plus the colored difference between $u_{0}$
 and $u_{T}$. (d) The groundtruth interpolation at time $T/2$ from the Middlebury datasets. (e)  The generated interpolation at time $T/2$ by segregation loop I. (f)  The absolute difference between (d) and (e). (g) The generated interpolation at time $T/2$ by segregation loop II. (h) The absolute difference between (d) and (g).}
\label{fig:2}       % Give a unique label
\end{figure*}

\section{Conclusion and Outlooking}
\label{sec:7}
The approach to image sequence interpolation by optimal control of a
transport equation has proven to be useful and competitive to existing
methods. While we started to model the images in $BV$ we ended up with
an algorithm which does not exploit this regularity but merely uses
the $L^2$-structure. This was due to the fact that one needs
Lipschitz-continuous flow fields to preserve
$BV$-regularity~\cite{colombini04}. Hence, we finally used $H^1$ flow
fields. However, this still imposes some regularity on the flow field
and discontinuous flow fields are still not allowed. In further work
it may be interesting to use $BV$ vector fields and hence try to
transport an image with a possibly discontinuous flow field.  Another
open question is, how to deal with objects which appear in the second
image but are not present in the first image. One possibility could be
to use heuristic techniques to estimate motions which occlude or
disclose objects as described in~\cite{stich08}.


\begin{thebibliography}{10}
\providecommand{\url}[1]{{#1}}
\providecommand{\urlprefix}{URL }
\expandafter\ifx\csname urlstyle\endcsname\relax
  \providecommand{\doi}[1]{DOI~\discretionary{}{}{}#1}\else
  \providecommand{\doi}{DOI~\discretionary{}{}{}\begingroup
  \urlstyle{rm}\Url}\fi

\bibitem{adams03}
Adams, R.A., Fournier, J.J.: Sobolev Spaces.
\newblock Academic Press (2003)

\bibitem{ambro08}
Ambrosio, L., Crippa, G., Lellis, C.D., Otto, F., Westdickenberg, M.: Transport
  Equations and Multi-D Hyperbolic Conservation Laws.
\newblock Springer (2008)

\bibitem{ambro00}
Ambrosio, L., Fusco, N., Pallara, D.: Functions of Bounded Variation and Free
  Discontinuity Problems.
\newblock Clarendon Press Oxford (2000)

\bibitem{ambrosio}
Ambrosio, L., Tilli, P., Zambotti, L.: Introduzione alla teoria della misura ed
  alla probabilit{\'{a}}.
\newblock Lecture notes of a course given at the Scuola Normale Superiore,
  unpublished

\bibitem{attouch06}
Attouch, H., Buttazzo, G., Michaille, G.: Variational Analysis in Sobolev and
  BV Spaces.
\newblock SIAM (2006)

\bibitem{aubert99}
Aubert, G., Kornprobst, P.: A mathematical study of the relaxed optical flow
  problem in the space {$BV(\Omega)^{*}$}.
\newblock SIAM J. Math Anal. \textbf{30}(6), 1282--1308 (1999)

\bibitem{aubert02}
Aubert, G., Kornprobst, P.: Mathematical Problems in Image Processing.
\newblock Springer Verlag New York, LLC (2002)

\bibitem{BakerSLRBS07}
Baker, S., Scharstein, D., Lewis, J.P., Roth, S., Black, M.J., Szeliski, R.: A
  database and evaluation methodology for optical flow.
\newblock In: ICCV, pp. 1--8 (2007)

\bibitem{barron94_2}
Barron, J., Khurana, M.: Determining optical flow for large motions using
  parametric models in a hierarchical framework.
\newblock In: Vision Interface, pp. 47--56 (1994)

\bibitem{borz02}
Borz{\'{i}}, A., Ito, K., Kunisch, K.: Optimal control formulation for
  determining optical flow.
\newblock SIAM Journal of Scientific Computing \textbf{24}, 818--847 (2002)

\bibitem{fortin91}
Brezzi, F., Fortin, M.: Mixed and Hybrid Finite Element Methods.
\newblock Springer-Verlag (1991)

\bibitem{brox04}
Brox, T., Bruhn, A., Papenberg, N., Weickert, J.: High accuracy optical flow
  estimation based on a theory for warping.
\newblock In: Computer Vision - ECCV 2004, Lecture Notes in Computer Science,
  pp. 25--36. Springer (2004)

\bibitem{bruhn89}
Bruhn, A., Weickert, J., Schn{\"o}rr, C.: Lucas/{K}anade meets
  {H}orn/{S}chunck: combining local and global optical flow methods.
\newblock Int. J. Comput. Vision \textbf{61}(3), 211--231 (2005)

\bibitem{peter83}
Burt, P.J., Edward, Adelson, E.H.: The laplacian pyramid as a compact image
  code.
\newblock IEEE Transactions on Communications \textbf{31}, 532--540 (1983)

\bibitem{colombini04}
Colombini, F., Luo, T., Rauch, J.: Nearly lipschitzean divergence free
  transport propagates neither continuity nor {$BV$} regularity.
\newblock Comm. Math. Sci. \textbf{2}(2), 207--212 (2004)

\bibitem{crippa07}
Crippa, G.: The flow associated to weakly differentiable vector fields.
\newblock Ph.D. thesis, Universit{\"a}t Z{\"u}rich (2007)

\bibitem{quang05}
Dang, Q.A.: Using boundary-operator method for approximate solution of a
  boundary value problem (bvp) for triharmonic equation.
\newblock Vietnam Journal of Mahtematics \textbf{33}(1), 9--18 (2005)

\bibitem{diperna89}
DiPerna, R., Lions, J.: Ordinary differential equations, transport theory and
  {S}obolev spaces.
\newblock Inventiones mathematicae \textbf{98}, 511--547 (1989)

\bibitem{elman91}
Elman, H., Silvester, D., Wathen, A.: Finite Elements and Fast Iterative
  Solvers.
\newblock OXFORD (2005)

\bibitem{enkelmann88}
Enkelmann, W.: Investigation of multigrid algorithms for the estimation of
  optical flow fields in image sequences.
\newblock Computer Vision, Graphics, and Image Processing \textbf{43}, 150--177
  (1998)

\bibitem{evans92}
Evans, L.C., Gariepy, R.F.: Measure Theory and Fine Properties of Functions.
\newblock CRC Press (1992)

\bibitem{girault86}
Girault, V., Raviart, P.A.: Finite Element Methods for Navier-Stokes Equations.
\newblock Springer-Verlag Berlin Heidelberg (1986)

\bibitem{hartman02}
Hartman, P.: Ordinary Differential Equations, second edn.
\newblock SIAM (2002)

\bibitem{hinterberger01}
Hinterberger, W., Scherzer, O.: Models for image interpolation based on the
  optical flow.
\newblock Computing \textbf{66}, 231--247 (2001)

\bibitem{hirsch07}
Hirsch, C.: Numerical Computation of Internal {\&} External Flows.
\newblock ELSEVIER (2007)

\bibitem{horn81}
Horn, B.K., Schunck, B.G.: Determining optical flow.
\newblock Artificial Intelligence \textbf{17}, 185--203 (1981)

\bibitem{kameda07}
Kameda, Y., Imiya, A.: The {W}illiam {H}arvey code: Mathematical analysis of
  optical flow computation for cardiac motion.
\newblock Computational Imaging and Vision \textbf{36}, 81--104 (2007)

\bibitem{kuzmin04}
Kuzmin, D., Turek, S.: High-resolution {FEM-TVD} schemes based on a fully
  multidimensional flux limiter.
\newblock Journal of Computational Physics \textbf{198}, 131--158 (2004)

\bibitem{lions71}
Lions, J.L.: Optimal Control of Systems Governed by Partial Differential
  Equations.
\newblock Springer-Verlag (1971)

\bibitem{nagel83}
Nagel, H.: Constraints for the estimation of displacement vector fields from
  image sequences.
\newblock In: International Joint Conference on Artifical Intelligence, pp.
  156--160 (1983)

\bibitem{william07}
Press, W.H., Teukolsky, S.A., Vetterling, W.T., Flannery, B.P.: Numerical
  Recipes 3rd Edition: The Art of Scientific Computing.
\newblock Cambridge University Press (2007)

\bibitem{riley06}
Riley, K.F., Hobson, M.P., Bence, S.J., Bence, S.: Mathematical methods for
  physics and engineering.
\newblock Cambridge University Press (2006)

\bibitem{ruhnau06}
Ruhnau, P., Schn{\"o}rr, C.: Optical stokes flow estimation: an imaging-based
  control approach.
\newblock Journal Experiments in Fluids \textbf{42}(1), 61--78 (2006)

\bibitem{stich08}
Stich, T., Linz, C., Albuquerque, G., Magnor, M.: View and time interpolation
  in image space.
\newblock Pacific Graphics \textbf{27}(7), 1781--1787 (2008)

\bibitem{suter94}
Suter, D.: Mixed-finite element based motion estimation.
\newblock Innovation and Technology in Biology and Medicine \textbf{15}(3),
  292--307 (1994)

\bibitem{fred05}
Tr{\"o}ltzsch, F.: Optimale Steuerung partieller Differentialgleichungen.
\newblock Vieweg (2005)

\bibitem{watkin04}
Watkinson, J.: The MPEG Handbook, second edn.
\newblock Focal Press (2004)

\bibitem{wedel09}
Wedel, A., Pock, T., Zach, C., Bischof, H., Cremers, D.: An improved algorithm
  for {TV-L1} optical flow.
\newblock In: Statistical and Geometrical Approaches to Visual Motion Analysis:
  International Dagstuhl Seminar, Dagstuhl Castle, Germany, July 13-18, 2008.
  Revised Papers, pp. 23--45. Springer-Verlag, Berlin, Heidelberg (2009).
\newblock \doi{http://dx.doi.org/10.1007/978-3-642-03061-1_2}

\end{thebibliography}
\end{document}